\documentclass[runningheads]{llncs}

\usepackage{eccv}

\usepackage{eccvabbrv}

\usepackage{graphicx}
\usepackage{booktabs}

\usepackage[accsupp]{axessibility}  

\usepackage{hyperref}

\usepackage{orcidlink}
\usepackage{multirow}
\usepackage{adjustbox}
\usepackage{amsmath}
\usepackage{amssymb}
\usepackage{algorithmic}
\usepackage{algorithm}
\usepackage{mathtools}
\usepackage{bm}
\usepackage{tabu}
\usepackage{xcolor}
\usepackage{colortbl}
\usepackage{multirow}
\usepackage{wrapfig}
\usepackage{subcaption}
\usepackage{lipsum}
\usepackage{soul}
\usepackage{amsmath}
\usepackage{setspace}
\usepackage{makecell}

\newcommand{\paragrapha}[2][5pt]{\noindent\textbf{#2}}

\newtheorem{assumption}[theorem]{Assumption}
\newcommand{\bs}{\boldsymbol}
\newcommand{\dif}{\mathrm d}

\newcommand{\bigO}{\mathcal O}
\renewcommand{\l}{\lambda}

\DeclareMathOperator*{\argmin}{arg\,min}

\DeclareMathOperator{\Pred}{Predictor}
\DeclareMathOperator{\Corr}{Corrector}
\newcommand{\NFE}{\mathrm{NFE}}
\newcommand{\CFG}{\mathrm{CFG}}

\sethlcolor{Gray}

\newcommand{\textbfg}[1]{\cellcolor{Gray}\textbf{#1}}

\newcommand{\ours}{DC-Solver}
\definecolor{Gray}{gray}{0.9}
\newcolumntype{C}[1]{>{\centering\arraybackslash}p{#1}}

\begin{document}
\title{DC-Solver: Improving Predictor-Corrector Diffusion Sampler via Dynamic Compensation} 

\titlerunning{DC-Solver}

\author{Wenliang Zhao\orcidlink{0000-0002-0920-1576} \and Haolin Wang\orcidlink{0009-0003-8852-174X}
 \and Jie Zhou\orcidlink{0009-0009-6880-7058} \and Jiwen Lu\orcidlink{0000-0002-6121-5529}\thanks{Corresponding author}}

\authorrunning{W.~Zhao et al.}

\institute{Department of Automation, Tsinghua University, China \and
Beijing National Research Center for Information Science and Technology, China\\
\email{zhaowl20@mails.tsinghua.edu.cn, wanghowlin@gmail.com, \{jzhou,lujiwen\}@tsinghua.edu.cn
}
}

\maketitle

\begin{abstract}
Diffusion probabilistic models (DPMs) have shown remarkable performance in visual synthesis but are computationally expensive due to the need for multiple evaluations during the sampling. Recent predictor-corrector diffusion samplers have significantly reduced the required number of function evaluations (NFE), but inherently suffer from a misalignment issue caused by the extra corrector step, especially with a large classifier-free guidance scale (CFG). In this paper, we introduce a new fast DPM sampler called DC-Solver, which leverages dynamic compensation (DC) to mitigate the misalignment of the predictor-corrector samplers. The dynamic compensation is controlled by compensation ratios that are adaptive to the sampling steps and can be optimized on only 10 datapoints by pushing the sampling trajectory toward a ground truth trajectory. We further propose a cascade polynomial regression (CPR) which can instantly predict the compensation ratios on unseen sampling configurations. Additionally, we find that the proposed dynamic compensation can also serve as a plug-and-play module to boost the performance of predictor-only samplers. Extensive experiments on both unconditional sampling and conditional sampling demonstrate that our DC-Solver can consistently improve the sampling quality over previous methods on different DPMs with a wide range of resolutions up to 1024$\times$1024. Notably, we achieve 10.38 FID (NFE=5) on unconditional FFHQ and 0.394 MSE (NFE=5, CFG=7.5) on Stable-Diffusion-2.1. Code is available at \url{https://github.com/wl-zhao/DC-Solver}.

\end{abstract}

\begin{figure}[!t]
    \centering
    \includegraphics[width=\linewidth]{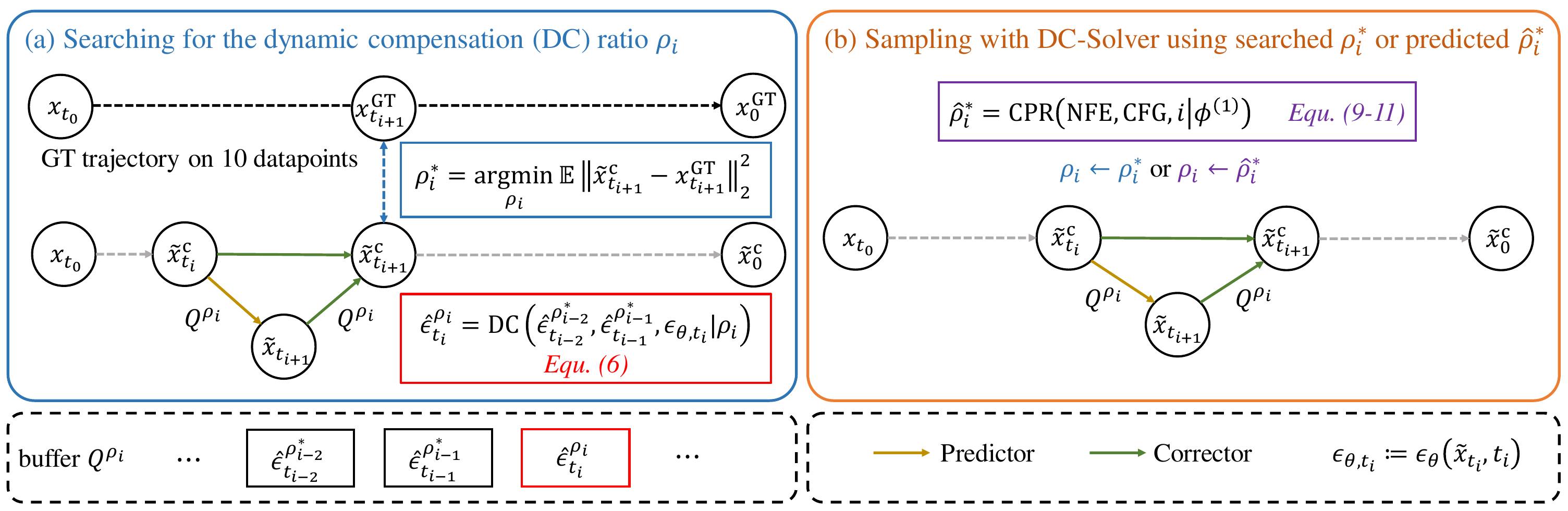}
    \caption{\textbf{The main idea of \textit{DC-Solver}.} \textbf{(a) Searching.} We propose dynamic compensation (DC) to mitigate the misalignment issue in the predictor-corrector diffusion sampler. The compensation is controlled by the ratios $\{\rho_i\}$ which are adaptive to the sampling step and can be optimized by pushing the sampling trajectory toward the ground truth trajectory on only 10 datapoints. \textbf{(b) Sampling.} The compensation ratios can be either efficiently searched as in (a) or instantly predicted by the cascade polynomial regression (CPR) given the desired NFE and CFG.}
    \label{fig:teaser}
\end{figure}

\section{Introduction}

Diffusion probabilistic models (DPMs)~\cite{sohl2015deep,ho2020denoising,song2021score,rombach2022high} have emerged as the new state-of-the-art generative models, demonstrating remarkable quality in various visual synthesis tasks~\cite{dhariwal2021diffusion,ho2022video,nichol2021glide,gu2022vector,zhang2023adding,mou2023t2i,ruiz2023dreambooth,liu2023zero,poole2022dreamfusion,wang2023prolificdreamer,mokady2023null,hertz2022prompt,parmar2023zero,gal2022image,shi2023dragdiffusion,meng2021sdedit,rombach2022high,brooks2023instructpix2pix}. Recent advances in large-scale pre-training of DPMs on image-text pairs also allow the generation of high-fidelity images given the text prompts~\cite{rombach2022high}. However, sampling from DPMs requires gradually performing denoising from Gaussian noises, leading to multiple evaluations of the denoising network $\bs\epsilon_\theta$, which is computationally expensive and time-consuming. Therefore, it is of great interest to design fast samplers of DPMs~\cite{zhang2022fast_deis,lu2022dpmsolver,lu2022dpmsolverpp,zhao2023unipc} to improve the sampling quality with few numbers of function evaluations (NFE).

Recent efforts on accelerating the sampling of DPMs can be roughly divided into training-based methods~\cite{salimans2022progressive,watson2021learning,nichol2021improved,song2023consistency,liu2022flow} and training-free methods~\cite{song2020denoising_ddim,lu2022dpmsolver,lu2022dpmsolverpp, zhang2022fast_deis,liu2022pseudo, zhang2022gddim,zhao2023unipc}. The latter families of approaches are generally preferred in applications because they can be applied to any pre-trained DPMs without the need for fine-tuning or distilling the denoising network. Modern training-free DPM samplers~\cite{lu2022dpmsolver,lu2022dpmsolverpp,zhang2022fast_deis,zhao2023unipc} mainly focus on solving the diffusion ODE instead of SDE~\cite{ho2020denoising, song2021score,bao2022analytic,zhang2022gddim}, since the stochasticity would deteriorate the sampling quality with few NFE. Specifically, \cite{lu2022dpmsolverpp,zhang2022fast_deis} adopt the exponential integrator~\cite{hochbruck_ostermann_2010} to significantly reduce the approximation error of the sampling process. More recently, Zhao~\etal~\cite{zhao2023unipc} proposed a predictor-corrector framework called UniPC, which can enhance the sampling quality without extra model evaluations. However, the extra corrector step will cause a misalignment between the intermediate corrected result $\tilde{\bs x}_{t_i}^{\rm c}$ and the reused model output $\bs\epsilon_\theta(\tilde{\bs x}_{t_i}, t_i)$. The influence of the misalignment has been witnessed in an analysis of UniPC~\cite{zhao2023unipc}, and it has been proven that re-computing the $\bs\epsilon_\theta(\tilde{\bs x}_{t_i}^{\rm c}, t_i)$ to ensure the alignment is indeed beneficial. However, naively re-computing $\bs\epsilon_\theta(\tilde{\bs x}_{t_i}^{\rm c}, t_i)$ would bring extra evaluations of the $\bs\epsilon_\theta$ and double the total computational costs.

In this paper, we propose a new fast sampler for DPMs called DC-Solver, which leverages dynamic compensation (DC) to mitigate the misalignment issue in the predictor-corrector framework. Specifically, we adopt the Lagrange interpolation of previous model outputs at a new timestep, which is controlled by a learned compensation ratio $\rho_i^*$. The compensation ratios are optimized by minimizing the $\ell_2$-distance between the intermediate sampling results and a ground truth trajectory, which can be achieved in less than 5min on only 10 datapoints. By examining the learned compensation ratios on different numbers of function evaluations (NFE) and classifier-free guidance scale (CFG), we further propose a cascade polynomial regression (CPR) that can instantly predict the desired compensation ratios on unseen NFE/CFG. Equipped with CPR, our DC-Solver allows users to freely adjust the configurations of CFG/NFE and substantially accelerates the sampling process. We also illustrate our method in~\Cref{fig:teaser}.

We perform extensive experiments on both unconditional sampling and conditional sampling tasks, where we show that DC-Solver consistently outperforms previous methods by large margins in 5$\sim$10 NFE. In the experiments on the state-of-the-art Stable-Diffusion~\cite{rombach2022high} (SD), we find DC-Solver can obtain the best sampling quality on different CFG (1.5$\sim$7.5), NFE (5$\sim$10) and pre-trained models (SD1.4, SD1.5, SD2.1, SDXL). Notably, DC-Solver achieves 0.394 MSE on SD2.1 with a guidance scale of 7.5 and only 5 NFE. By performing the cascade polynomial regression to the compensation ratios searched on only a few configurations, our DC-Solver can generalize to unseen NFE/CFG and surpass previous methods. Besides, we find the proposed dynamic compensation can also serve as a plug-and-play component to boost the performance of predictor-only solvers like~\cite{song2020denoising_ddim,lu2022dpmsolverpp}. We provide some qualitative comparisons between our DC-Solver and previous methods in~\Cref{fig:viz}, where it can be clearly observed that DC-Solver can generate high-resolution and photo-realistic images with more details in only 5 NFE.

\captionsetup[subfigure]{justification=centering}
\begin{figure*}[!t]
\begin{subfigure}{0.25\linewidth}
\centering
\includegraphics[width=.95\linewidth]{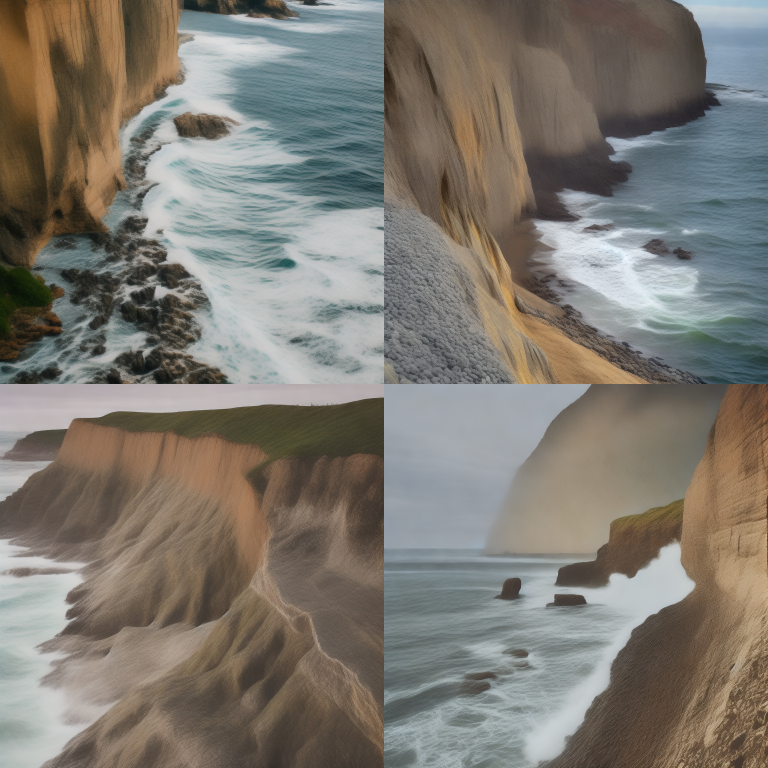}
\caption{DPM-Solver++~\cite{lu2022dpmsolverpp} \\ (MSE 0.443)}
\end{subfigure}%
\begin{subfigure}{0.25\linewidth}
\centering
\includegraphics[width=.95\linewidth]{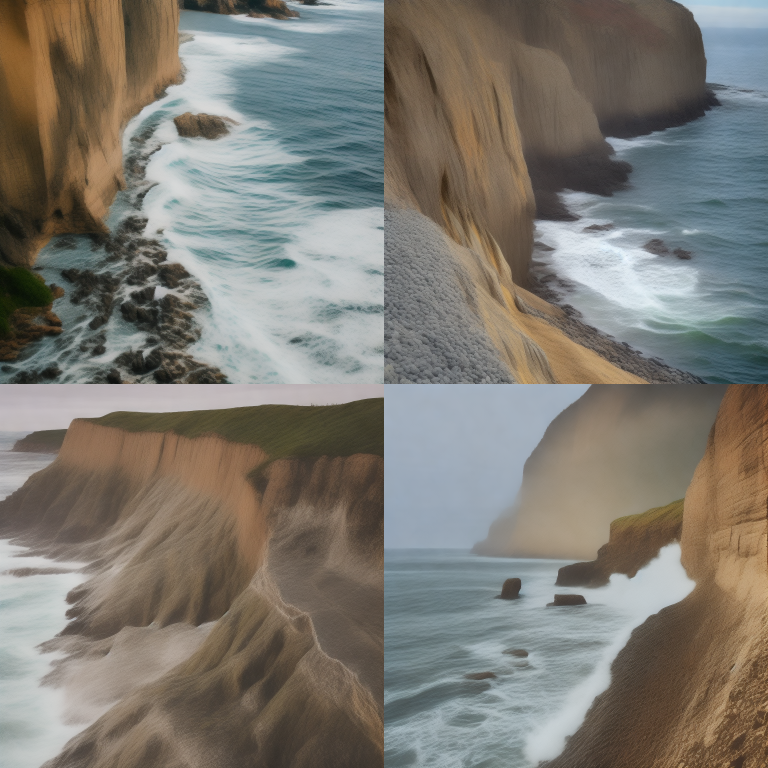}
\centering
\caption{DEIS~\cite{zhang2022fast_deis}\\ (MSE 0.436)}
\end{subfigure}%
\begin{subfigure}{.25\linewidth}
\centering
\includegraphics[width=.95\linewidth]{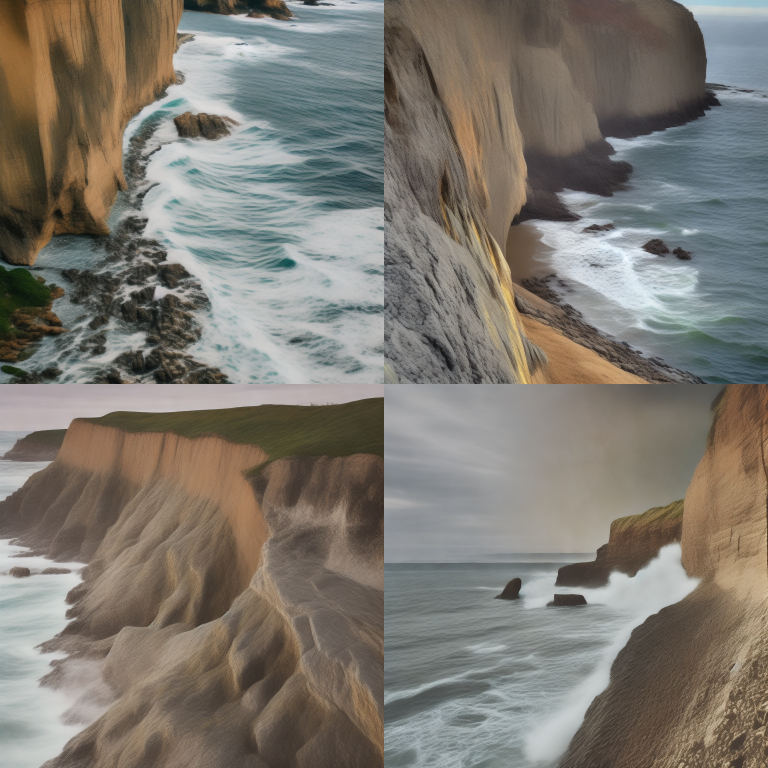}
\caption{UniPC~\cite{zhao2023unipc} \\(MSE 0.434)}
\end{subfigure}%
\begin{subfigure}{.25\linewidth}
\centering
\includegraphics[width=.95\linewidth]{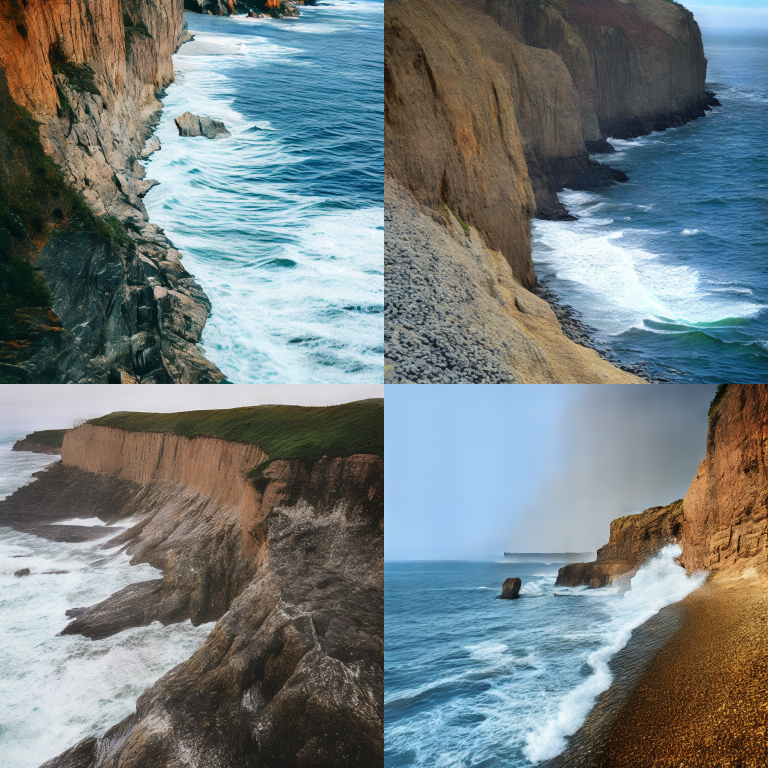}
\caption{DC-Solver (\textbf{Ours}) \\(MSE 0.394)}
\end{subfigure}%
\caption{\textbf{Qualitative comparisons on Stable-Diffusion-2.1.} Images above are sampled from SD2.1 (768$\times$768) using the text prompt ``\textit{A photo of a serene coastal cliff with waves crashing against the rocks below}" with a classifier-free guidance scale of 7.5 and only \textbf{5} number of function evaluations (NFE).  We provide the generated images from 4 random initial noises for each method. We show that DC-Solver is able to generate high-resolution and photo-realistic images with more details. Best viewed in color.
}\label{fig:viz}
\end{figure*}

\section{Related Work}
\paragrapha{Diffusion probabilistic models.} Diffusion probabilistic models (DPMs), originally proposed in~\cite{sohl2015deep,ho2020denoising,song2021score}, have demonstrated impressive ability in high-fidelity visual synthesis. The basic idea of DPMs is to train a denoising network $\bs\epsilon_\theta$ to learn the reverse of a Markovian diffusion process~\cite{ho2020denoising} through score-matching~\cite{song2021score}. To reduce the computational costs in high-resolution image generation and add more controllability, Rombach~\etal~\cite{rombach2022high} propose to learn a DPM on latent space and adopt the cross-attention~\cite{vaswani2017attention} to inject conditioning inputs. Based on the latent diffusion models~\cite{rombach2022high}, a series of more powerful DPMs called Stable-Diffusion~\cite{rombach2022high} are released, which are trained on a large-scale text-image dataset LAION-5B~\cite{schuhmann2021laion} and soon become famous for the high-resolution text-to-image generation. In practical usage, classifier-free guidance~\cite{ho2022classifier} (CFG) is usually adopted to encourage the adherence between the text prompt and the generated image. Despite the impressive synthesis quality of DPMs, they suffer from heavy computational costs during the inference due to the need for multiple evaluations of the denoising network. In this paper, we focus on designing a fast sampler that can accelerate the sampling process of a wide range of DPMs and is suitable to different CFG, thus promoting the application of DPMs.

\paragrapha{Fast DPM samplers.} Developing fast samplers for DPMs has gained increasing attraction since the prevailing of Stable Diffusion~\cite{rombach2022high}. Modern fast samplers of DPMs usually work by discretizing the diffusion ODE or SDE. Among those, ODE-based methods~\cite{song2020denoising_ddim,lu2022dpmsolver,lu2022dpmsolverpp,zhao2023unipc} are shown to be more effective in few-step sampling due to the absence of stochasticity. The widely used DDIM~\cite{song2020denoising_ddim} can be viewed as a 1-order approximation of the diffusion ODE. DPM-Solver~\cite{lu2022dpmsolver} and DEIS~\cite{zhang2022fast_deis} adopt exponential integrator to develop high-order solvers and significantly reduce the sampling error. DPM-Solver++~\cite{lu2022dpmsolverpp} investigates the data-prediction parameterization and multistep high-order solver which are proven to be useful in practice, especially for conditional sampling. UniPC~\cite{zhao2023unipc} borrows the merits of the predictor-corrector paradigm~\cite{hochbruck2005explicit} in numeral analysis and finds the corrector can substantially improve the sampling quality in the few-step sampling. However, UniPC~\cite{zhao2023unipc} suffers from a misalignment issue caused by the extra corrector step, which is observed also and mentioned in their original paper. In this work, we aim to mitigate the misalignment through a newly proposed approach called dynamic compensation.

\section{Method}
\subsection{Preliminaries: Fast Sampling of DPMs}
We start by briefly reviewing the basic ideas of diffusion probabilistic models (DPMs) and how to efficiently sample from them. DPMs aim to model the data distribution $q_0(\bs x_0)$ by learning the reverse of a forward diffusion process. Given the noise schedule $\{\alpha_t,\sigma_t\}_{t=0}^{T}$, the diffusion process gradually adds noise to a clean data point $\bs x_0$ and the equivalent transition can be computed by $\bs x_t = \alpha_t\bs x_0 + \sigma_t\bs \epsilon, \bs\epsilon \in \mathcal{N}(\bs 0, \bs I)$, and the resulting distribution $q_T(\bs x_T)$ is approximately Gaussian. During training, a network $\bs \epsilon_\theta$ is learned to perform score matching~\cite{batzolis2021conditional} by estimating the $\bs \epsilon$ given the current $\bs x_t$, timestep $t$ and the condition $c$. Specifically, the training objective is to minimize:
\begin{equation}
    \mathbb{E}_{\bs{x}_0, \bs{\epsilon}, t}\left[w(t)\|\bs\epsilon_\theta (\bs x_t, t, c) - \bs\epsilon\|^2_2\right].
\end{equation}
The above simple objective makes it more stable to train DPMs on large-scale image-text pairs and enables the generation of high-fidelity visual content. However, sampling from DPMs is computationally expensive due to the need for multiple evaluations of the denoising network $\bs\epsilon_\theta$ (\eg, 200 steps for DDIM~\cite{song2020denoising_ddim}). 

Modern fast samplers for DPMs~\cite{lu2022dpmsolver,lu2022dpmsolverpp,zhang2022fast_deis} significantly reduce the required number of function evaluations (NFE) by solving the diffusion ODE with a multistep paradigm, which leverages the model outputs of previous points to improve convergence. Recently, UniPC~\cite{zhao2023unipc} proposes to use a corrector to refine the result at each sampling step, which can further improve the sampling quality. Denote the sampling timesteps as $\{t_i\}_{i=0}^M$ and let $Q$ be the buffer to store previous model outputs of the denoising network, the update logic of modern samplers of DPMs from $t_{i-1}$ to $t_i$ can be summarized as follows:
\begin{align}
     &\tilde{\bs x}_{t_i} \leftarrow \Pred(\tilde{\bs x}_{t_{i-1}}^{\rm c}, Q),\label{equ:predictor_step}\\
     &\tilde{\bs x}_{t_i}^{\rm c} \leftarrow \Corr(\tilde{\bs x}_{t_i},  \bs\epsilon_\theta(\tilde{\bs x}_{t_i}, t_i), Q)\quad(\text{optional})\label{equ:corrector_step}\\
     &Q \overset{\mathrm{buffer}}{\leftarrow} \bs\epsilon_\theta(\tilde{\bs x}_{t_i}, t_i),\label{equ:buffer_Q}
\end{align}
where $\tilde{\bs x}_{t_i}^{\rm c}$ denote the refined result after the corrector and $\tilde{\bs x}_{t_i}^{\rm c}=\tilde{\bs x}_{t_i}$ if no corrector is used as in~\cite{zhang2022fast_deis,lu2022dpmsolverpp}.

\subsection{Better Alignment via Dynamic Compensation}\label{sec:dc}
Although the extra corrector step~\eqref{equ:corrector_step} can improve the theoretical convergence order, there exists a misalignment between $\tilde{\bs x}_{t_i}^{\rm c}$ and $\bs\epsilon_\theta(\tilde{\bs x}_{t_i}, t_i)$, \ie, the $\bs\epsilon_\theta(\tilde{\bs x}_{t_i}, t_i)$ pushed into the buffer $Q$ is not computed from the corrected intermediate result $\tilde{\bs x}_{t_i}^{\rm c}$. It is also witnessed in~\cite{zhao2023unipc} that replacing the $\bs\epsilon_\theta(\tilde{\bs x}_{t_i}, t_i)$ with $\bs\epsilon_\theta(\tilde{\bs x}_{t_i}^c, t_i)$ (which would bring an extra forward of $\bs\epsilon_\theta$) can further improve the sampling quality. The effects of the misalignment will be further amplified by the large guidance scale in the widely used classifier-free guidance~\cite{ho2022classifier} (CFG) for conditional sampling:
\begin{equation}
    \bar{\bs\epsilon}_\theta (\bs x_t, t, c) = s\cdot \bs\epsilon_\theta (\bs x_t, t, c) + (1-s)\cdot \bs\epsilon_\theta (\bs x_t, t, \O),\label{equ:conditional}
\end{equation}
where $s>1$ is the guidance scale and $s=7.5$ is usually adopted in text-to-image synthesis on Stable-Diffusion~\cite{rombach2022high}.

\begin{figure*}
\begin{minipage}{.49\textwidth}
\begin{algorithm}[H]
\setstretch{1.053}
\caption{Searching.}
\label{alg:search}
\begin{algorithmic}
\small
   \STATE {\bfseries Require:} current timestep $t_i$, a ground truth trajectory ${\bs x}_t^{\mathrm{GT}, N}$, the (corrected) intermediate results $\tilde{\bs x}_{t_i}^{\mathrm{c}, N}$, a buffer $Q$, learning rate $\alpha$, number of iterations  $L$.

    \STATE $\rho_i \leftarrow 1.0$, $Q^{\rm copy}\leftarrow Q$
    \FOR{$l=1$ {\bfseries to} $L$}
    \STATE compute $\hat{\bs\epsilon}^{\rho_i}(\tilde{\bs x}_{t_i}^{\mathrm{c}, N}, t_i)$ via~\eqref{equ:dynamic_compensation}
    \STATE $Q^{\rho_i} \leftarrow [Q_{[:-1]}^{\rm copy}, \hat{\bs\epsilon}^{\rho_i}(\tilde{\bs x}_{t_i}^{\mathrm{c}, N}, t_i)]$
    \STATE $\tilde{\bs x}_{t_{i+1}}^N \leftarrow {\rm Pred}(\tilde{\bs x}_{t_{i}}^{\mathrm{c}, N}, Q^{\rho_i})$
    \STATE ${\bs x}_{t_{i+1}}^{\mathrm{c}, N} \leftarrow {\rm Corr}(\tilde{\bs x}_{t_{i+1}}^N,  \bs\epsilon_\theta(\tilde{\bs x}_{t_i}^N, t_i), Q^{\rho_i})$
    \STATE $\rho_i \leftarrow \rho_i - \alpha \nabla_{\rho_i} \|{\bs x}_{t_{i+1}}^{\mathrm{c}, N} - {\bs x}_t^{\mathrm{GT}, N}\|_2^2$
    \ENDFOR
    \STATE {\bfseries return:} $\rho_i$, $Q^{\rho_i}$
\end{algorithmic}
\end{algorithm}
   \vspace{1pt}
\end{minipage}\hfill
\begin{minipage}{.49\textwidth}
\begin{algorithm}[H]
\setstretch{1.0825}
\caption{Sampling.}
\label{alg:sampling}
\begin{algorithmic}
   \small
   \STATE {\bfseries Require:} sampling timesteps $\{t_i\}_{i=0}^{M}$, initial noise $\tilde{\bs x}_{t_0}^{\mathrm{c}}\sim \mathcal{N}(\bs 0, \bs I)$, compensation ratios $\{\rho_i^*\}_{i=0}^{M-1}$ either searched by~\eqref{equ:search} or directly predicted by~\eqref{equ:polyregress}.

   \FOR{$i=0$ {\bfseries to} $M-1$}
   \IF{$i \ge K$} 
        \STATE compute $\hat{\bs\epsilon}^{\rho_i^*}(\tilde{\bs x}_{t_i}^{\mathrm{c}}, t_i)$ via~\eqref{equ:dynamic_compensation}
       \STATE $Q \leftarrow [Q_{[:-1]}, \hat{\bs\epsilon}^{\rho_i^*}(\tilde{\bs x}_{t_i}^{\mathrm{c}}, t_i)]$
   \ENDIF
       \STATE $\tilde{\bs x}_{t_{i+1}} \leftarrow {\rm Pred}(\tilde{\bs x}_{t_{i}}^{\mathrm{c}}, Q)$
       \STATE ${\bs x}_{t_{i+1}}^{\mathrm{c}} \leftarrow {\rm Corr}(\tilde{\bs x}_{t_{i+1}},  \bs\epsilon_\theta(\tilde{\bs x}_{t_i}, t_i), Q)$
   \ENDFOR

   \STATE {\bfseries return:} ${\bs x}_{t_{M}}^{\mathrm{c}}$

\end{algorithmic}
\end{algorithm}
\vspace{1pt}
\end{minipage}
\vspace{-30pt}
\end{figure*}

\paragrapha{Dynamic compensation.} The aforementioned misalignment issue motivates us to seek for a better method to approximate $\bs\epsilon_\theta(\tilde{\bs x}_{t_i}^c, t_i)$ after~\eqref{equ:corrector_step} with no extra NFE. To achieve this, we propose a new method called dynamic compensation (DC) that leverages the previous model outputs stored in the buffer $Q$ to approach the target $\bs\epsilon_\theta(\tilde{\bs x}_{t_i}^c, t_i)$. Given a ratio $\rho_i$, let $t_i' = \rho_i t_{i} + (1 - \rho_i)t_{i-1}$, we adopt the following estimation based on Lagrange interpolation:
\begin{align}
    &\hat{\bs\epsilon}^{\rho_i}(\tilde{\bs x}_{t_i}^c, t_i) = \sum_{k=0}^{K}\prod_{\substack{0\le l\le K\\l\neq k}}\frac{t'_i - t_{i-l}}{t_{i-k} - t_{i-l}}\bs\epsilon_{\theta}(\tilde{\bs x}_{t_{i-k}}, t_{i-k}),\label{equ:dynamic_compensation}
\end{align}
where $K$ represents the order of the Lagrange interpolation and $\{\bs\epsilon_{\theta}(\tilde{\bs x}_{t_{i-k}}, t_{i-k})\}_{k=0}^K$ are previous model outputs retrieved from buffer $Q$. The above estimation is then used to replace the last item in $Q$ to obtain a new buffer:
\begin{equation}
Q^{\rho_i} \leftarrow [Q_{[:-1]}, \hat{\bs\epsilon}^{\rho_i}(\tilde{\bs x}_{t_i}^c, t_i)],
\end{equation}
where $Q_{[:-1]}$ denotes the elements in $Q$ except the last one. Note that when $\rho_i=1.0$ we have $\hat{\bs\epsilon}^{\rho_i}(\tilde{\bs x}_{t_i}^c, t_i)=\bs\epsilon_\theta(\tilde{\bs x}_{t_{i}}, t_{i})$, which implies that the buffer $Q$ is not updated. By varying the $\rho_i$, we can obtain a trajectory of $\hat{\bs\epsilon}^{\rho_i}(\tilde{\bs x}_{t_i}^c, t_i)$ and our goal is to find an optimal $\rho_i^*$ which can minimize the local error to push the sampling trajectory toward the ground truth trajectory. Since the optimal compensation ratio $\rho_i^*$ is different across the sampling timesteps, we name our method dynamic compensation.

\paragrapha{Searching for the optimal $\rho_i^*$.} The optimal compensation ratios $\{\rho_i^*\}$ can be viewed as learnable parameters and optimized through backpropagation. Given a DPM, we first obtain ground truth trajectories $\{\bs x_t^{{\rm GT}}\}$ of $N$ initial noises. During each sampling step, we minimize the following objective:
\begin{equation}
    \rho_i^* = \argmin_{\rho_i} \mathbb{E}\|\tilde{\bs x}_{t_{i+1}}^{\rm c}(\tilde{\bs x}_{t_{i}}^{\rm c}, Q^{\rho_i}) - \bs x_{t_{i+1}}^{{\rm GT}}\|_2^2,\label{equ:search}
\end{equation}
where $\tilde{\bs x}_{t_{i+1}}^{\rm c}$ is computed similar to~\eqref{equ:predictor_step} and \eqref{equ:corrector_step}, and the expectation is approximated over the $N$ datapoints. The above objective ensures that the local approximation error on the selected $N$ datapoints is reduced with an optimal compensation ratio $\rho^*_i$. We find in our experiments that $N=10$ is sufficient in order to learn the optimal $\{\rho^*_i\}_{i=1}^M$ which also works well on any other initial noises. Besides, we show that both the local and global convergence of DC-Solver are guaranteed under mild conditions (see Supplementary). When an optimal $\rho_i^*$ is searched, we replace the buffer $Q$ with $Q^{\rho_i^*}$ and move to the next sampling step. We also list the detailed searching procedure in \Cref{alg:search}.

\begin{figure}[!t]
    \centering
    \begin{subfigure}{\linewidth}
    \includegraphics[width=.33\linewidth]{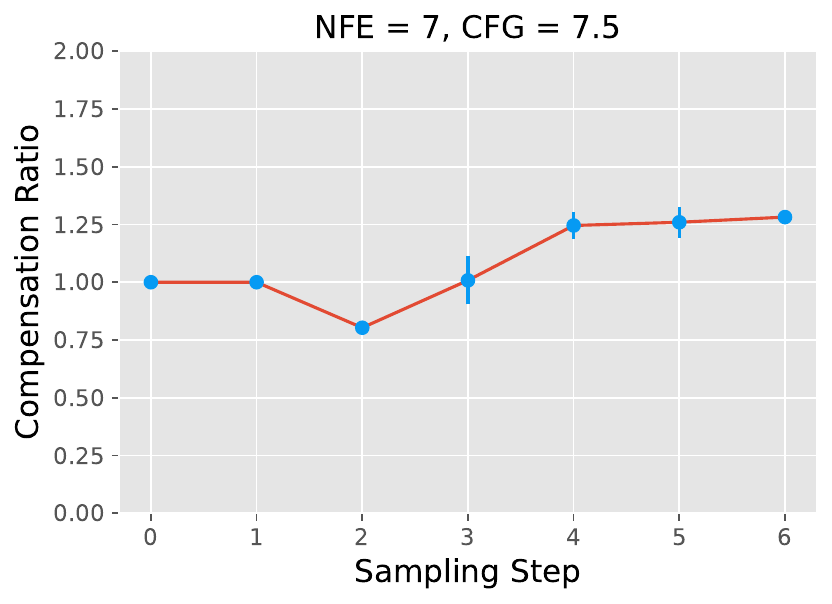}%
    \includegraphics[width=.33\linewidth]{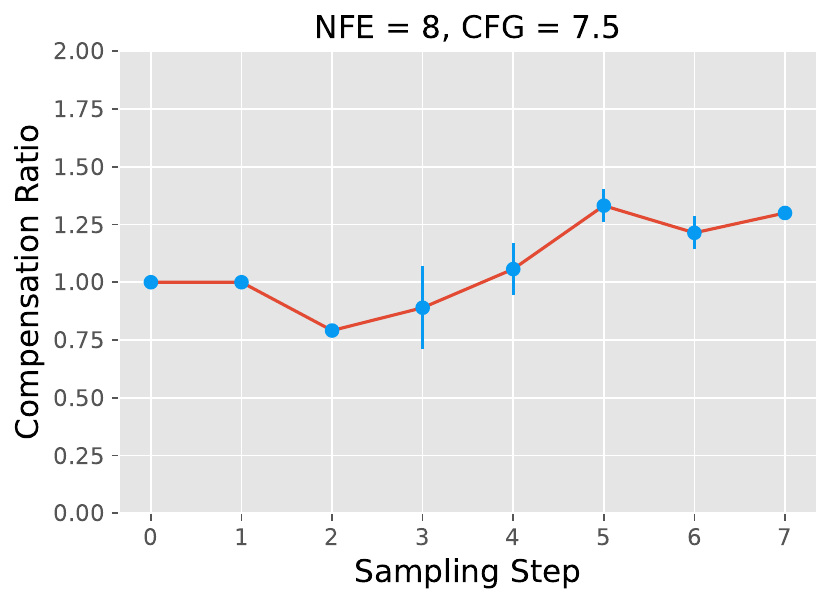}%
    \includegraphics[width=.33\linewidth]{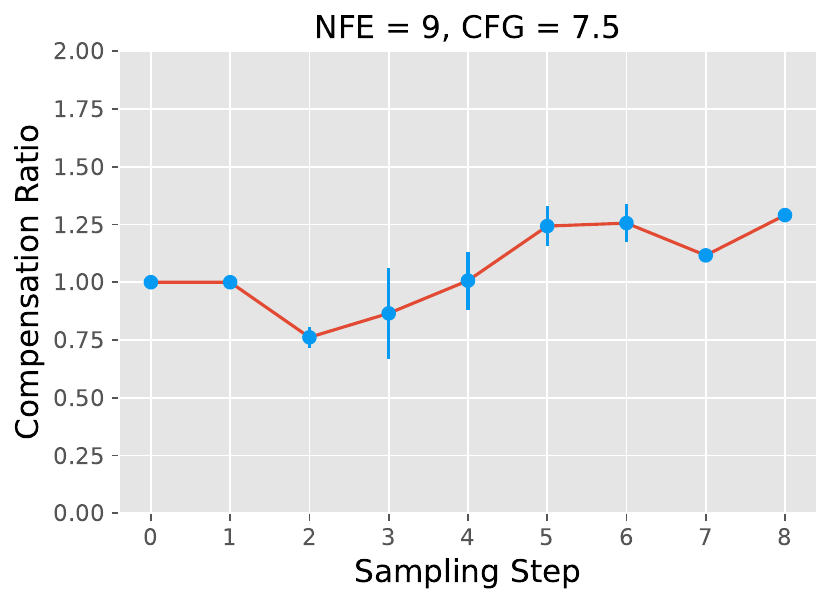}
    \caption{$\CFG =7.5, \NFE\in[7, 8, 9]$}
    \end{subfigure}
    \begin{subfigure}{\linewidth}
    \includegraphics[width=.33\linewidth]{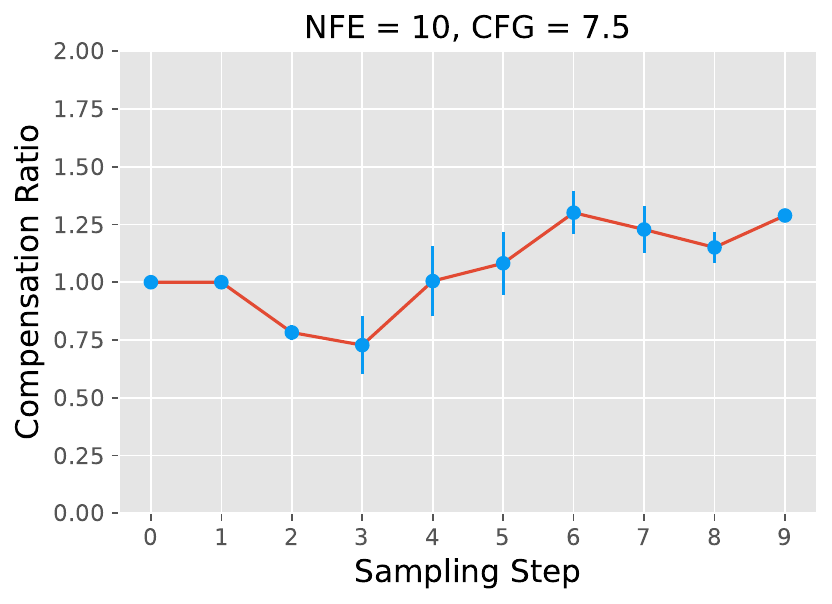}%
    \includegraphics[width=.33\linewidth]{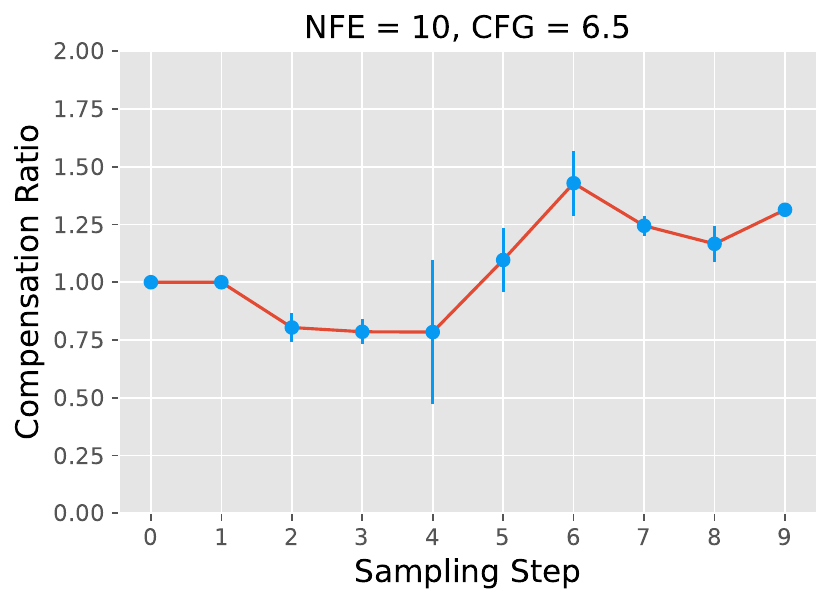}%
    \includegraphics[width=.33\linewidth]{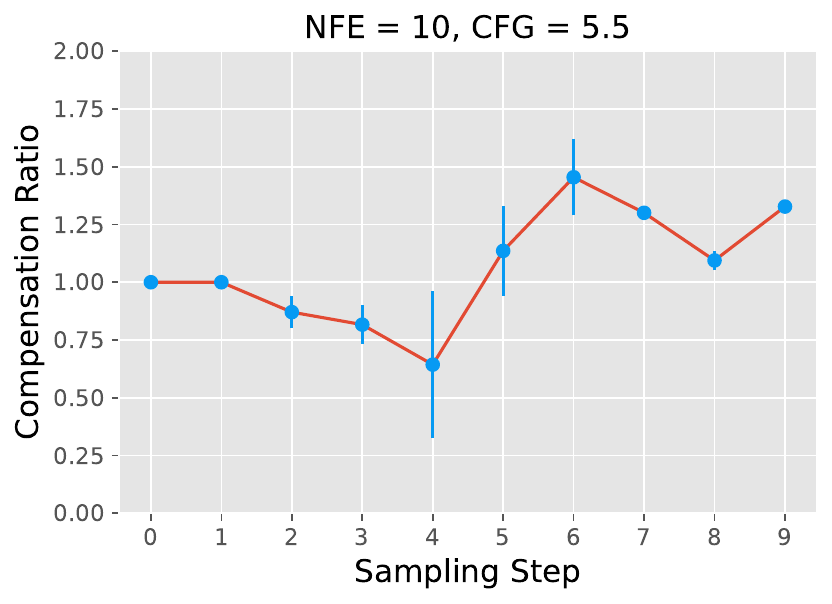}%
    \caption{$\NFE = 10, \CFG\in[7.5, 6.5, 5.5]$}
    \end{subfigure}
    \caption{\textbf{Relationship between compensation ratios and CFG/NFE.}  We adopt the widely used Stable-Diffusion-1.5~\cite{rombach2022high} and search for the optimal compensation ratios for different CFG and NFE and find that the compensation ratios evolve continuously with the variations in CFG/NFE.}
    \label{fig:ratio_relation}
\end{figure}

\paragrapha{Sampling with DC-Solver.} After obtaining the optimal compensation ratios $\{\rho_i^*\}$, we can directly apply them in our DC-Solver to sample from the pre-trained DPM. Similar to the searching stage, we update the buffer with $Q^{\rho_i^*}$ after each sampling step to improve the alignment between the intermediate result and the model output (see \Cref{alg:sampling} for details). Note that the dynamic compensation~\eqref{equ:dynamic_compensation} does not introduce any extra NFE, thus the overall computational costs are almost unchanged.

\subsection{Generalization to Unseen NFE \& CFG}\label{sec:generalization}
Although the compensation ratio $\rho_i^*$ can be obtained via~\eqref{equ:search}, the optimization still requires extra time costs (about 1min for NFE=5). Since the $\rho_i^*$ is specifically optimized for a diffusion ODE, the optimal choice for $\rho_i^*$ is different when NFE or CFG varies. This issue would limit the application of conditional sampling~\eqref{equ:conditional}, where the users may try different combinations of NFEs and CFGs. Therefore, it is vital to design a method to estimate the optimal compensation ratios without extra time costs of searching. To this end, we propose a technique called cascade polynomial regression that can instantly compute the desired compensation ratios given the CFG and NFE.

\paragrapha{Cascade polynomial regression.} To investigate how to efficiently estimate the compensation ratios, we start by searching for the optimal compensation ratios on the widely used Stable-Diffusion-1.5~\cite{rombach2022high} for different configurations of CFG and NFE and plot the relationship between the compensation ratios and CFG/NFE in \Cref{fig:ratio_relation}. For each configuration, we perform the search for 10 runs and report the averaged results as well as the corresponding standard deviation. Our key observation is that the learned optimal compensation ratios evolve almost continuously when CFG/NFE changes. Inspired by the shapes of the curves in~\Cref{fig:ratio_relation}, we propose a cascade polynomial regression to directly predict the compensation ratios. Formally, define the $p$-order polynomial with the coefficients $\bs\phi\in \mathbb{R}^{p+1}$ as $f^{(p)}(a | \bs\phi) = \sum_{j=0}^p \bs\phi_j a^j$, we predict the compensation ratios as follows:
\begin{align}
    \bs\phi^{(2)}_{j,k} &= f_1^{(p_1)}(\NFE | \bs\phi^{(1)}_{j,k}),  0\le j \le p_3, 0\le k \le p_2 \label{equ:phi1}\\
    \bs\phi^{(3)}_j &= f_2^{(p_2)}(\CFG | \bs\phi^{(2)}_j), 0\le j\le p_3\\
    \hat{\rho}_i^* &= f_3^{(p_3)}(i|\bs\phi^{(3)}), 2\le i\le \NFE - 1\label{equ:polyregress}
\end{align}
The above formulation indicates that we model the change of compensation ratios \wrt sampling steps via a polynomial, whose coefficients are determined by the CFG, NFE, and the $\bs\phi^{(1)}\in\mathbb{R}^{(p_3+1)\times (p_2 + 1)\times (p_1 + 1)}$. As we will show in \Cref{sec:analysis}, $\bs\phi^{(1)}$ can be obtained by applying the off-the-shelf regression toolbox (such as \texttt{curve\_fit} in \texttt{scipy}) on the pre-computed optimal compensation ratios of few configurations of NFE/CFG. With cascade polynomial regression, we can efficiently compute the compensation ratios with neglectable extra costs, making our DC-Solver more practical in real applications.

\subsection{Discussion}
Recently, a concurrent work DPM-Solver-v3~\cite{zheng2023dpmv3} proposes to learn several coefficients called empirical model statistics (EMS) of the pre-trained model to obtain a better parameterization during sampling. Our DC-Solver has several distinctive advantages: 1) DPM-Solver-v3 requires extensive computational resources to optimize and save the EMS parameters (\eg, 1024 datapoints, $~$11h on 8 GPUs, $~$125MB disk space), while our DC-Solver only needs a \textit{scalar} compensation ratio $\rho_i$ for each step and can be searched more efficiently in both time and memory (10 datapoints, $<$5min on a single GPU). 2) The EMS is specific to different CFG, and adjusting CFG requires another training of EMS to obtain good results. Our DC-Sovler adopts cascade polynomial regression to predict the desired compensation ratios on unseen CFG/NFE \textit{instantly}. 3) Our proposed dynamic compensation is a more general technique that can boost the performance of both predictor-only and predictor-corrector samplers.

\begin{figure*}[!ht]
    \centering
    \begin{subfigure}{.32\linewidth}
    \includegraphics[width=\linewidth]{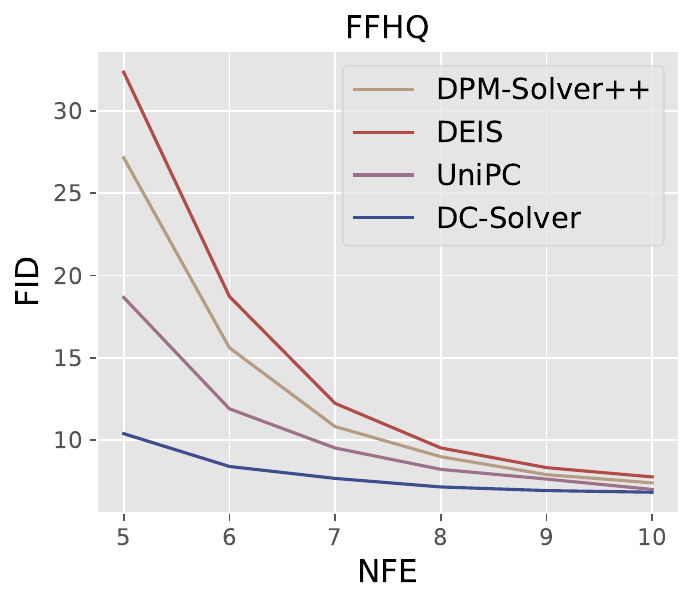}
    \label{fig:cifar}
    \end{subfigure}\hfill
    \begin{subfigure}{.32\linewidth}
    \includegraphics[width=\linewidth]{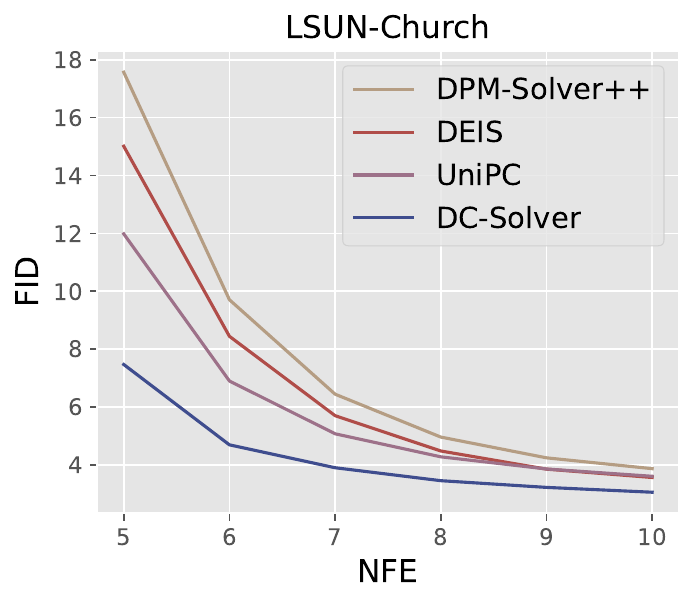}
    \label{fig:ffhq}
    \end{subfigure}\hfill
    \begin{subfigure}{.32\linewidth}
    \includegraphics[width=\linewidth]{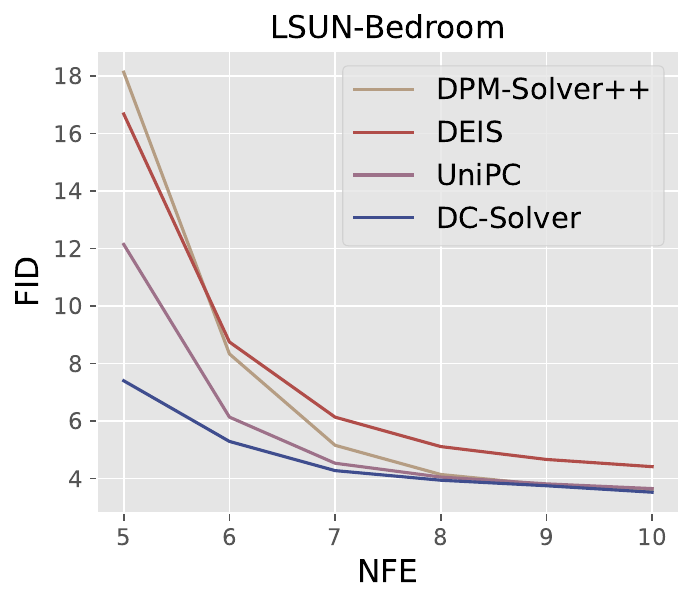}
    \label{fig:lsun}
    \end{subfigure}
    \caption{\textbf{Unconditional sampling results.} We compare our \ours{} with previous methods on FFHQ~\cite{karras2019ffhq}, LSUN-Church~\cite{yu2015lsun}, and LSUN-Bedroom~\cite{yu2015lsun}. The FID$\downarrow$ on different numbers of function evaluations (NFE) is used to measure the sampling quality. We show that \ours{} significantly outperforms other methods, especially with few NFE.
    }
    \label{fig:main_uncond}
\end{figure*}

\section{Experiments}

\subsection{Implementation Details}
Our DC-Solver follows the predictor-corrector paradigm by applying the dynamic compensation to UniPC~\cite{zhao2023unipc}. We set $K=2$ in~\eqref{equ:dynamic_compensation} and skip the compensation when $i<K$, which is equivalent to $\rho_0=\rho_1=1.0$. During the searching stage, we set the number of datapoints $N=10$. We use a 999-step DDIM~\cite{song2020denoising_ddim} to generate the ground truth trajectory $\bs x_t^{\rm GT}$ in the conditional sampling while we found a 200-step DDIM is enough for unconditional sampling. We use AdamW~\cite{adamw} to optimize the compensation ratios for only $L=40$ iterations, which can be finished in 5min on a single GPU. We use $p_1=p_2=2$ and $p_3=4$ for the cascade polynomial regression.

\subsection{Main Results}\label{sec:main_results}

We perform extensive experiments on both unconditional and conditional sampling on different datasets to evaluate our DC-Solver. Following common practice~\cite{lu2022dpmsolverpp,zhao2023unipc}, we use FID$\downarrow$ of the generated images in unconditional sampling and MSE$\downarrow$ between the generated latents and the ground truth latents on 10K prompts in conditional sampling. Our experiments demonstrate that our DC-Solver achieves better sampling quality than previous methods including DPM-Solver++~\cite{lu2022dpmsolverpp}, DEIS~\cite{zhang2022fast_deis} and UniPC~\cite{zhao2023unipc} both qualitatively and quantitatively.

\paragrapha{Unconditional sampling.} We start by comparing the unconditional sampling quality of different methods. We adopt the widely used latent-diffusion models~\cite{rombach2022high} pre-trained on FFHQ~\cite{karras2019ffhq}, LSUN-Bedroom~\cite{yu2015lsun}, and LSUN-Church~\cite{yu2015lsun}. We use the 3-order version for all the methods and report the FID$\downarrow$ on 5$\sim$10 NFE, as shown in \Cref{fig:main_uncond}. We find our DC-Solver consistently outperforms previous methods on different datasets. With the dynamic compensation, DC-Solver improves over UniPC significantly, especially with fewer NFE. Compared with UniPC, DC-Solver reduces the FID by 8.28, 4,51, 4.75 on FFHQ, LSUN-Church, and LSUN-Bedroom respectively when NFE=5.

\begin{figure*}[!ht]
    \centering
    \begin{subfigure}{.32\linewidth}
    \includegraphics[width=\linewidth]{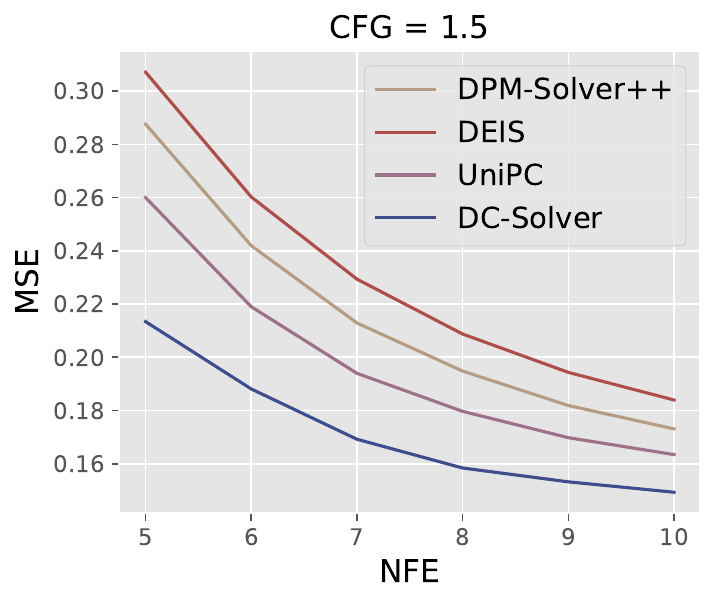}
    \label{fig:cond_s1.5}
    \end{subfigure}\hfill
    \begin{subfigure}{.32\linewidth}
    \includegraphics[width=\linewidth]{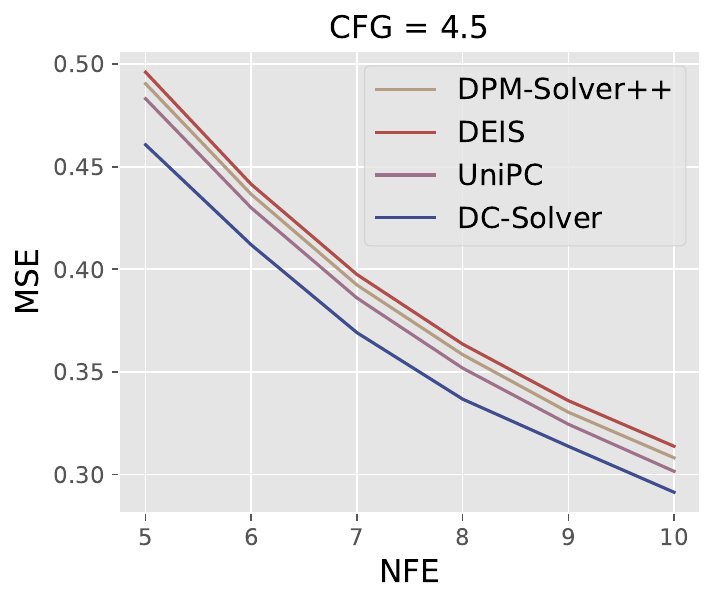}
    \label{fig:cond_s4.5}
    \end{subfigure}\hfill
    \begin{subfigure}{.32\linewidth}
    \includegraphics[width=\linewidth]{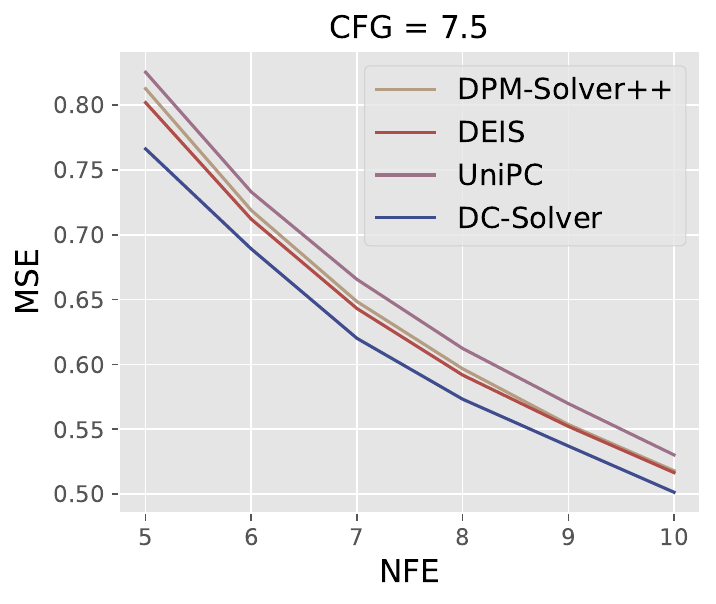}
    \label{fig:cond_s7.5}
    \end{subfigure}
    \caption{\textbf{Conditional sampling results.} We compare the sampling quality of different methods using the Stable-Diffusion-1.5 with classifier-free guidance (CFG) varying from 1.5 to 7.5. The sampling quality is measured by the mean squared error (MSE$\downarrow$) between the generated latents and the ground truth latents obtained by a 999-step DDIM. We randomly select 10K captions from MS-COCO2014 as the text prompts. We observe that \ours{} consistently achieves better sampling quality on different NFE/CFG.
    }
    \label{fig:main_cond}
\end{figure*}

\paragrapha{Conditional sampling.} We conduct experiments on Stable-Diffusion-1.5~\cite{rombach2022high} to compare the conditional sampling performance of different methods. Following common practice~\cite{lu2022dpmsolverpp,zhao2023unipc}, we report the mean squared error (MSE) between the generated latents and the ground truth latents (obtained by a 999-step DDIM~\cite{song2020denoising_ddim}) on 10K samples. The input prompts for the diffusion models are randomly sampled from MS-COCO2014 validation dataset~\cite{lin2014microsoft}. Apart from the default guidance scale CFG for Stable-Diffusion-1.5, we also conducted experiments with CFG=1.5/4.5. The results in \Cref{fig:main_cond} demonstrate that our DC-Solver achieves the lowest MSE on all of the three guidance scales. Notably, we find that the performance enhancement over UniPC achieved by DC-Solver surpasses the differences observed among those three previous methods.

\subsection{Ablation study} 

We conduct ablation studies on the design of our method and the hyper-parameters on FFHQ~\cite{karras2019ffhq}. The comparisons of the sampling quality measured by FID$\downarrow$ of different configurations are summarized in \Cref{tab:ablations}.

\paragrapha{Compensation methods.} Firstly, we evaluate the effectiveness of the proposed dynamic compensation in~\Cref{abl:compensation_method}. We start from the baseline method UniPC~\cite{zhao2023unipc} and apply different compensation methods. As discussed in~\Cref{sec:dc}, the baseline with no compensation is equivalent to $\rho_i\equiv 1.0, \forall i$. We then conduct experiments by setting $\rho_i$ to other constants, \ie, $\rho_i\equiv 0.9$ or $\rho_i\equiv 1.1$, which also corresponds to performing interpolation or extrapolation in~\eqref{equ:dynamic_compensation}. Since the compensation ratio is constant across the sampling steps, we call these ``static compensation''. We find that adjusting the $\rho_i$ can indeed influence the performance significantly, and the static compensation with $\rho_i\equiv 1.1$ outperforms the baseline method. As shown in the last row, our proposed dynamic compensation further improves the sampling quality by large margins. 

\paragrapha{Number of datapoints.} We investigate how the number of datapoints would affect the performance of our DC-Solver. We compare the sampling quality when using 5,10,20,30 datapoints and list the results in~\Cref{abl:datapoints}. We also provide the memory costs during the searching stage. We demonstrate that $N=10$ is enough to obtain satisfactory results while further increasing the number of datapoints will not bring significant improvement.

\paragrapha[3pt]{Order of dynamic compensation.} According to~\eqref{equ:dynamic_compensation}, the order $K$ controls how the $\hat{\bs\epsilon}^{\rho_i}(\tilde{\bs x}_{t_i}^c, t_i)$ varies with $\rho_i$. The results in~\Cref{abl:dc_order} indicate that $K=2$ can produce the best sampling quality, indicating that performing Lagrange interpolation on a parabola-like trajectory is the optimal choice.

\paragrapha{Number of optimization iterations.} We now examine how many iterations are required to learn the dynamic compensation ratios. In~\Cref{abl:iterations}, we report the FID of different optimization iterations as well as the time costs for each sampling step. We find the optimization converges after about 40 iterations. In this case, the actual time cost for each NFE is around $(\NFE - 2)\times22.2{\rm s}$ since we do not need to learn for the first two steps ($\rho_0=\rho_1=1.0$). Note that the time costs in the searching stage will not affect the inference speed since we can directly predict the compensation ratios using the CPR described in~\Cref{sec:generalization}.

\begin{table*}[t]
\caption{\textbf{Ablation studies.} We perform ablation studies on the design of our method and the hyper-parameters. Sampling quality is measured by FID$\downarrow$ on FFHQ~\cite{karras2019ffhq}. The configurations with the best trade-offs are selected and highlighted in {\hl{gray}}. 
}\label{tab:ablations}
\centering
\begin{subtable}{.476\textwidth}
\caption{Compensation method.}\label{abl:compensation_method}
\adjustbox{width=\linewidth}{
\begin{tabular}{p{4.0 cm}*{4}{C{1.0cm}}}\toprule
\multicolumn{1}{l}{\multirow{2}[0]{*}{Compensation Method}} & \multicolumn{4}{c}{NFE} \\\cmidrule{2-5}
      & 5     & 6     & 8     & 10 \\\midrule
Baseline~\cite{zhao2023unipc}  &   18.66    &    11.89   &   8.21    &  6.99\\
Static ($\rho_i\equiv 0.9$) &  26.43   &  16.50   &   9.84  & 7.84 \\
Static ($\rho_i\equiv 1.1$) &  13.99   &  10.21  &  7.86   & 6.90 \\
\rowcolor{Gray} Dynamic ($\rho_i=\rho^*_i$) &  10.38 &  8.39   &   7.14 & 6.82 \\\bottomrule
\end{tabular}%
}
\end{subtable}%
\begin{subtable}{.51\textwidth}
\caption{Number of datapoints.}\label{abl:datapoints}
\adjustbox{width=\linewidth}{
\begin{tabular}{C{2.0cm}C{2.4cm}*{4}{C{1.0cm}}}\toprule
\multicolumn{1}{c}{\multirow{2}[0]{*}{\#Datapoints}} & \multicolumn{1}{c}{Memory} & \multicolumn{4}{c}{NFE}  \\\cmidrule{3-6}
      &    (GB)   & 5     & 6     & 8     & 10    \\\midrule
5     &   9.15 & 12.39 & 9.79  &   7.05  &  6.84 \\
\rowcolor{Gray} 10    &   12.10 &  10.38  &  8.39 & 7.14  & 6.82 \\
20    &   18.61 &  10.37 &  8.31 &  7.01 & 6.63 \\
30    &   22.44  &  10.93  &  8.40 &  6.95  & 6.70 \\\bottomrule
\end{tabular}
}
\end{subtable}

\begin{subtable}{.476\textwidth}
\caption{Order of dynamic compensation.}\label{abl:dc_order}
\adjustbox{width=\linewidth}{
\begin{tabular}{C{4.0cm}*{4}{C{1.0cm}}}\toprule
\multirow{2}[0]{*}{DC Order $K$} & \multicolumn{4}{c}{NFE} \\\cmidrule{2-5}
      & 5     & 6     & 8     & 10 \\\midrule
1 & 12.70  &   9.44   & 7.07   &  6.55 \\
\rowcolor{Gray} 2 &  10.38  &  8.39 & 7.14  & 6.82 \\
3 & 11.63  &   8.89   & 6.98   &  6.72 \\\bottomrule
\end{tabular}%
}
\end{subtable}%
\begin{subtable}{.51\textwidth}
\caption{Number of optimization iterations.}\label{abl:iterations}
\adjustbox{width=\linewidth}{
\begin{tabular}{C{2.0cm}C{2.4cm}*{4}{C{1.04cm}}}\toprule
\multicolumn{1}{c}{\multirow{2}[0]{*}{\#Iterations}} & \multicolumn{1}{c}{{Time}} & \multicolumn{4}{c}{NFE}        \\\cmidrule{3-6}
      &  (s) & 5     & 6     & 8     & 10  \\\midrule
20    &  11.4 &   11.34 & 8.69 &  6.96 &  6.55 \\
\rowcolor{Gray} 40    &  22.2 &  10.38  &  8.39 & 7.14  & 6.82 \\
60    &  33.4 &  10.63 & 8.38 &  7.00 &  6.65 \\\bottomrule
\end{tabular}%
}
\end{subtable}
\end{table*}

\subsection{More Analyses}\label{sec:analysis}
In this section, we will provide in-depth analyses of DC-Solver, including some favorable properties and more quantitative/qualitative results.

\paragrapha{Comparisons with different pre-trained DPMs.} In our main results~\Cref{sec:main_results}, we have evaluated the effectiveness of DC-Solver on conditional sampling using Stable-Diffusion-1.5. We now provide comparisons on more different pre-trained DPMs in~\Cref{tab:sd_version}, where we report the MSE between the generated latents to the ground truth similar to~\Cref{fig:main_cond}. Specifically, we consider three versions of Stable-Diffusion (SD): 1) SD1.4 is the previous version of SD1.5, which is widely used in~\cite{lu2022dpmsolverpp,zhao2023unipc} to evaluate the conditional sampling quality; 2) SD2.1 is trained using another parameterization called $v$-prediction~\cite{salimans2022progressive} and can generate 768$\times$768 images; 3) SDXL is the latest Stable-Diffusion model that can generate realistic images of 1024$\times$1024. Note that we use the default CFG for all the models (CFG=7.5 for SD1.4 and SD2.1, CFG=5.0 for SDXL). We demonstrate that DC-Solver consistently outperforms previous methods with 5$\sim$10 NFE, indicating that our method has a wide application and can be applied to any pre-trained DPMs to accelerate the sampling.

\begin{table}[!t]
  \centering
  \caption{\textbf{Comparisons with different DPMs.} We compare the sampling quality between DC-Solver and previous methods using different pre-trained Stable-Diffusion (SD) models including SD1.4, SD2.1, and SDXL, which can generate images of various resolutions from 512$\times $512 to 1024$\times$1024. We compare the MSE$\downarrow$ with 5$\sim$10 NFE with the default classifier-free guidance scale of each model. We show that our \ours{} consistently outperforms previous methods by large margins.}
  \adjustbox{width=0.9\linewidth}{
    \begin{tabular}{p{4.0cm}*{6}{C{1.5cm}}}\toprule
    \multicolumn{1}{l}{\multirow{2}[0]{*}{Method}} & \multicolumn{6}{c}{NFE} \\\cmidrule{2-7}
          & 5     & 6     & 7     & 8     & 9     & 10 \\\toprule
    \multicolumn{7}{l}{\textit{SD1.4, $\epsilon$-prediction, CFG=7.5, 512$\times$512}} \\\midrule
    DPM-Solver++~\cite{lu2022dpmsolverpp} & 0.803  & 0.711  & 0.642  & 0.590  & 0.547  & 0.510  \\
    DEIS~\cite{zhang2022fast_deis}  & 0.795  & 0.706  & 0.636  & 0.586  & 0.544  & 0.508  \\
    UniPC~\cite{zhao2023unipc} & 0.813  & 0.724  & 0.658  & 0.607  & 0.563  & 0.525  \\
    \rowcolor{Gray} DC-Solver (Ours) & \textbf{0.760}  & \textbf{0.684}  & \textbf{0.615}  & \textbf{0.565}  & \textbf{0.527}  & \textbf{0.496}  \\
     \toprule
    \multicolumn{7}{l}{\textit{SD2.1, $v$-prediction, CFG=7.5, 768$\times$768}} \\\midrule
    DPM-Solver++~\cite{lu2022dpmsolverpp} & 0.443  & 0.421  & 0.404  & 0.390  & 0.379  & 0.370  \\
    DEIS~\cite{zhang2022fast_deis}  & 0.436  & 0.416  & 0.400  & 0.387  & 0.376  & 0.368  \\
    UniPC~\cite{zhao2023unipc} & 0.434  & 0.415  & 0.400  & 0.390  & 0.381  & 0.373  \\
    \rowcolor{Gray} DC-Solver (Ours) & \textbf{0.394}  & \textbf{0.364}  & \textbf{0.336}  & \textbf{0.309}  & \textbf{0.315}  & \textbf{0.294}  \\\midrule
    \multicolumn{7}{l}{\textit{SDXL, $\epsilon$-prediction, CFG=5.0, 1024$\times$1024}}\\\midrule   
    DPM-Solver++~\cite{lu2022dpmsolverpp} & 0.745  & 0.659  & 0.601  & 0.558  & 0.527  & 0.502 \\
    DEIS~\cite{zhang2022fast_deis}  &  0.778  & 0.683  & 0.619  & 0.571  & 0.538  & 0.511 \\
    UniPC~\cite{zhao2023unipc} & 0.718  & 0.645  & 0.593  & 0.553  & 0.524  & 0.500 \\
    \rowcolor{Gray} DC-Solver (Ours) & \textbf{0.689}  & \textbf{0.626}  & \textbf{0.574}  & \textbf{0.529}  & \textbf{0.510}  & \textbf{0.487} \\\bottomrule
    \end{tabular}%
}
\label{tab:sd_version}%
\end{table}

\begin{table}
  \centering
  \caption{\textbf{Generalization to unseen NFE \& CFG.} By performing the cascade polynomial regression to the compensation ratios searched on $\CFG\in[1.5, 4.5, 7.5, 10.5]$ and $\NFE\in[10, 15, 20]$, our DC-Solver can generalize to unseen NFE and CFG and outperform previous methods by large margins. The sampling quality is measured by the MSE$\downarrow$ between the generated latents and the ground truth on SD2.1~\cite{rombach2022high}.} \label{tab:unseen}
    \adjustbox{width=0.9\linewidth}{
    \begin{tabular}{C{2.0cm}p{4.0cm}*{5}{C{1.5cm}}}\toprule
    \multirow{2}[0]{*}{CFG} & \multicolumn{1}{c}{\multirow{2}[0]{*}{Method}} & \multicolumn{4}{c}{NFE} \\\cmidrule{3-6}
          &       & 12    & 14    & 16    & 18 \\\midrule
    \multirow{4}[0]{*}{3.0} & DPM-Solver++~\cite{lu2022dpmsolverpp}~~~~~~ & 0.212  & 0.209  & 0.198  & 0.196  \\
          & DEIS~\cite{zhang2022fast_deis} & 0.215  & 0.210  & 0.199  & 0.198  \\
          & UniPC~\cite{zhao2023unipc} & 0.211  & 0.208  & 0.206  & 0.205  \\
          &\cellcolor{Gray}DC-Solver (Ours) & \textbfg{0.103}  & \textbfg{0.093}  & \textbfg{0.087}  & \textbfg{0.083} \\\midrule
    \multirow{4}[0]{*}{6.0} & DPM-Solver++\cite{lu2022dpmsolverpp} & 0.312  & 0.304  & 0.293  & 0.289  \\
          & DEIS~\cite{zhang2022fast_deis} & 0.312  & 0.305  & 0.293  & 0.290 \\
          & UniPC~\cite{zhao2023unipc} & 0.311  & 0.304  & 0.298  & 0.296  \\
          & \cellcolor{Gray}DC-Solver (Ours) & \textbfg{0.215}  & \textbfg{0.196}  & \textbfg{0.182}  & \textbfg{0.169}  \\\midrule
    \multirow{4}[0]{*}{9.0} & DPM-Solver++\cite{lu2022dpmsolverpp} & 0.404  & 0.393  & 0.385  & 0.377  \\
          & DEIS~\cite{zhang2022fast_deis} & 0.402  & 0.391  & 0.380  & 0.374 \\
          & UniPC~\cite{zhao2023unipc} & 0.406  & 0.394  & 0.386  & 0.377  \\
          & \cellcolor{Gray}DC-Solver (Ours) & \textbfg{0.338}  & \textbfg{0.314}  & \textbfg{0.293}  & \textbfg{0.275}  \\\bottomrule
    \end{tabular}%
    }
\end{table}

\paragrapha{Generalization to unseen NFE \& CFG.} Based on the observation of the optimal compensation ratios and the proposed cascade polynomial regression (CPR) in~\Cref{sec:generalization}, our DC-Solver can be applied to unseen NFE and CFG without extra time costs for the searching stage. This is important because the users might frequently adjust the NFE and CFG to generate the desired images. To evaluate the effectiveness of the CPR, we first search the optimal compensation ratios for $\CFG\in[1.5, 4.5, 7.5, 10.5]$ and $\NFE\in[10, 15, 20]$ (which covers most of the use cases in real applications). We then use the \texttt{curve\_fit} in the \verb|scipy| library to obtain the $\bs\phi^{(1)}$ in~\eqref{equ:phi1} and predict the compensation ratios $\hat{\rho}_i^*$ on unseen configurations where $\CFG\in[3.0, 6.0, 9.0]$ and $\NFE\in [12, 14, 16, 18]$. The results of DC-Solver with the predicted compensation ratios on unseen NFE and CFG on SD2.1 can be found in~\Cref{tab:unseen}, where we also provide the results of previous methods~\cite{lu2022dpmsolverpp,zhang2022fast_deis,zhao2023unipc} for comparisons. We observe that DC-Solver with the compensation ratios predicted by CPR can still achieve lower MSE on all the unseen configurations. These results indicate that in order to use DC-Solver in real scenarios, we only need to perform CPR on sparsely selected configurations of CFG and NFE.

\paragrapha{Enhance any solver with dynamic compensation.} Although our DC-Solver was originally designed to mitigate the misalignment issue in the predictor-corrector frameworks, we will show that the dynamic compensation (DC) can also boost the performance of predictor-only DPM samplers. Similar to~\eqref{equ:search}, we can also search for an optimal $\rho_i^*$ to minimize $\|\tilde{\bs x}_{t_{i+1}}(\tilde{\bs x}_{t_{i}}, Q^{\rho_i}) - \bs x_{t_{i+1}}^{{\rm GT}}\|_2^2$. To verify this, we conduct experiments on DDIM~\cite{song2020denoising_ddim} and DPM-Solver++~\cite{lu2022dpmsolverpp} by applying the DC to them and the results are shown in~\Cref{tab:dc_any}. The FID$\downarrow$ on FFHQ~\cite{karras2019ffhq} is reported as the evaluation metric. We show that DC can significantly improve the sampling quality of the two baseline predictor-only solvers. These results indicate that our dynamic compensation can serve as a plug-and-play module to enhance any existing solvers of DPMs.

\paragrapha{Visualizations.} We now provide some qualitative comparisons between our DC-Solver and previous methods on SD2.1 with CFG=7.5 and NFE=5, as shown in~\Cref{fig:viz}. The images sampled from 4 random initial noises are displayed. We find that while other methods tend to produce blurred images with few NFE, our DC-Solver can generate photo-realistic images with more details.

\paragrapha{Inference speed and memory.} We compare the inference speed and memory of DC-Solver with previous methods, as shown in~\Cref{tab:speed}. For all the methods, we sample from the Stable-Diffusion-2.1~\cite{rombach2022high} using a single NVIDIA RTX 3090 GPU with a batch size of 1 and NFE=5/10/15. Our results show that DC-Solver achieves similar speed and memory to previous methods, indicating that DC-Solver can improve the sample quality without introducing noticeable extra computational costs during the inference.

\paragrapha{Limitations.} Despite the effectiveness of DC-Solver, it cannot be used with SDE-based samplers~\cite{xue2023sa} because of the stochasticity. How to apply DC-Solver to SDE samplers requires future investigation of a stochasticity-aware metric instead of the $\ell_2$-distance in~\eqref{equ:search}.

\begin{table}[!t]
  \centering
  \caption{\textbf{Applying DC to predictor-only solvers.} We compare the FID$\downarrow$ on FFHQ~\cite{karras2019ffhq} using two methods DDIM~\cite{song2020denoising_ddim} and DPM-Solver++~\cite{lu2022dpmsolverpp} as the baselines. We show that dynamic compensation (DC) can also significantly boost the performance of predictor-only solvers.}
  \adjustbox{width=0.9\linewidth}{
    \begin{tabular}{p{4.3cm}*{6}{C{1.5cm}}}\toprule
    \multicolumn{1}{l}{\multirow{2}[0]{*}{Method}} & \multicolumn{6}{c}{NFE} \\\cmidrule{2-7}
          & 5     & 6     & 7     & 8     & 9     & 10 \\\toprule
    DDIM~\cite{song2020denoising_ddim} & 57.92  & 42.67  & 32.82  & 26.96  & 23.25  & 19.09  \\
    \rowcolor{Gray} \quad + DC (Ours) & 16.56  & 15.50  & 12.51  & 11.33  & 9.62  & 9.21  \\\midrule
    DPM-Solver++~\cite{lu2022dpmsolverpp} & 27.80  & 16.01  & 11.16  & 9.17  & 8.04  & 7.40  \\
    \rowcolor{Gray}\quad + DC (Ours) & 11.97  & 8.64  & 7.70  & 7.32  & 7.10  & 6.94  \\\bottomrule
    \end{tabular}%
}
  \label{tab:dc_any}%
\end{table}
\begin{table}[!t]
  \centering
  \caption{\textbf{Comparisons of inference speed and memory.} We compare the inference speed and memory cost of different sampling methods with batch size 1 on SD2.1~\cite{rombach2022high} using a single NVIDIA RTX 3090 GPU. For inference time, we report the mean and std of 10 runs for each method and NFE. Our DC-Solver achieves similar speed to previous methods with the same NFE.}   \label{tab:speed}%
    \adjustbox{width=0.9\linewidth}{
    \begin{tabular}{p{3.5cm}*{4}{C{2.5cm}}}\toprule
    \multicolumn{1}{l}{\multirow{2}[0]{*}{Method}} & \multicolumn{1}{c}{{Memory}} & \multicolumn{3}{c}{Inference Time (s)} \\\cmidrule{3-5}
          &    (GB)   & \multicolumn{1}{c}{NFE = 5} & \multicolumn{1}{c}{NFE = 10} & \multicolumn{1}{c}{NFE = 15} \\\midrule
    DPM-Solver++~\cite{lu2022dpmsolverpp} & 14.21 & 1.515($\pm$0.003) & 2.833($\pm$0.007) & 4.168($\pm$0.005) \\
    UniPC~\cite{zhao2023unipc} & 14.37 & 1.533($\pm$0.004) & 2.865($\pm$0.004) & 4.203($\pm$0.003) \\
    \rowcolor{Gray} DC-Solver (Ours) & 14.37 & 1.532($\pm$0.003) & 2.867($\pm$0.005) & 4.203($\pm$0.004) \\\bottomrule
    \end{tabular}%
    }
\end{table}

\section{Conclusions}
In this paper, we have proposed a new fast sampler of DPMs called DC-Solver, which leverages the dynamic compensation to effectively mitigate the misalignment issue in previous predictor-corrector samplers. We have shown that the optimal compensation ratios can be either searched efficiently using only 10 datapoints on a single GPU in 5min, or instantly predicted by the proposed cascade polynomial regression on unseen CFG/NFE. Extensive experiments have demonstrated that DC-Solver significantly outperforms previous methods in 5$\sim$10 NFE, and can be applied to different pre-trained DPMs including SDXL. We have also found that the proposed dynamic compensation can also serve as a plug-and-play module to boost the performance of predictor-only methods. We hope our investigation on dynamic compensation can inspire more effective approaches in the few-step sampling of DPMs.

\section*{Acknowledgements}
This work was supported in part by the National Key Research and Development Program of China under Grant 2022ZD0160102, and in part by the National Natural Science Foundation of China under Grant 62125603, Grant 62321005, Grant 62336004.

\begin{appendix}
\section{Detailed Background of Diffusion Models}
\subsection{Diffusion Models}
In this section, we will provide a detailed background of diffusion probabilistic models (DPMs)~\cite{ho2020denoising,song2021score}. DPMs usually contain a forward diffusion process that gradually adds noise to the clean data and a backward denoising process that progressively removes the noise to obtain the cleaned data. The diffusion process can be defined either discretely~\cite{ho2020denoising} or continuously~\cite{song2021score}. We will focus on the latter since continuous DPMs are usually used in the context of DPM samplers~\cite{lu2022dpmsolver,lu2022dpmsolverpp,zhao2023unipc}. Let $\bs x_0$ be a random variable from the data distribution $q_0(\bs x_0)$, the forward (diffusion) process gradually adds noise via:
\begin{equation}
    q_{t|0}(\bs x_t | \bs x_0) = \mathcal{N}(\bs x_t | \alpha_t \bs x_0, \sigma_t^2 \bs I),
\end{equation}
where $\alpha_t,\sigma_t$ control the noise schedule and the signal-to-noise-ratio $\alpha_t^2/\sigma_t^2$ is decreasing w.r.t $t$. The noise schedule is designed such that the resulting distribution $q_T(\bs x_T)$ is approximately Gaussian. The forward process can be also formulated via an SDE~\cite{kingma2021variational}:

\begin{equation}
    \dif \bs x_t  = f(t) \bs x_t \dif t  + g(t)\dif \bs w_t,\quad \bs x_0\sim q_0(\bs x_0)\label{equ:sde}
\end{equation}
where $f(t) = \frac{\dif \log \alpha_t}{\dif t}$, $g^2(t)=\frac{\dif \sigma_t^2}{\dif t} - 2\frac{\dif \log \alpha_t}{\dif t}\sigma_t^2$ and $\bs w_t$ is the standard Wiener process. The reverse process can be analytically computed under some conditons~\cite{song2021score}:

\begin{align}
  \dif \bs x_t = [f(t)\bs x_t - g^2(t)\nabla_{\bs x} \log q_t(\bs x_t)]\dif t + g(t)\dif \bar{\bs w}_t,
\end{align}
where $ \bar{\bs w}_t$ is the standard Winer process in the reverse time. DPM is trained to estimate the scaled score function $-\sigma_t\nabla_{\bs x} \log q_t(\bs x_t)$ via a neural network $\bs \epsilon_\theta$, and the corresponding SDE during sampling is 
\begin{equation}
    \dif \bs x_t = \left[f(t)\bs x_t + \frac{g^2(t)}{\sigma_t}\bs \epsilon_\theta(\bs x_t, t)\right]\dif t + g(t)\dif \bar{\bs w}_t. \label{equ:reverse_sde}
\end{equation}

\subsection{ODE-based DPM samplers}
Although one can numerally solve the diffusion SDE by discretizing~\eqref{equ:reverse_sde}, the stochasticity would harm the sampling quality especially when the step size is large. On the contrary, the probability flow ODE~\cite{song2021score} is more practical:
\begin{equation}
    \frac{\dif \bs x_t}{\dif t} = f(t) \bs x_t - \frac{g^2(t)}{2}\nabla_{\bs x}\log q_t(\bs x_t).
\end{equation}
Modern fast samplers of DPMs~\cite{lu2022dpmsolver,lu2022dpmsolverpp,zhao2023unipc} aim to efficiently solve the above ODE with small numbers of function evaluations (NFE) by introducing several useful techniques such as the exponential integrator~\cite{lu2022dpmsolver,zhang2022fast_deis}, the multi-step method~\cite{lu2022dpmsolverpp,zhang2022fast_deis}, data-prediction~\cite{lu2022dpmsolverpp}, and predictor-corrector paradigm~\cite{zhao2023unipc}. For example, the deterministic version of DDIM~\cite{song2020denoising_ddim} can be viewed as a 1-order discretization of the diffusion probability flow ODE. DPM-Solver~\cite{lu2022dpmsolver} leverages an insightful parameterization (logSNR) and exponential integrator to achieve a high-order solver. DPM-Solver++~\cite{lu2022dpmsolverpp} further adopts the multi-step method to estimate high-order derivatives. Specifically, one can use a buffer to store the outputs of $\bs\epsilon_\theta$ on previous points and use them to increase the order of accuracy. PNDM~\cite{liu2022pseudo} modified classical multi-step numerical methods to corresponding pseudo numerical methods for DPM sampling. UniPC~\cite{zhao2023unipc} introduces a predictor-corrector framework that also uses the model output at the current point to improve the sampling quality, and bypasses the extra model evaluations by re-using the model outputs at the next sampling step.  Generally speaking, the formulation of existing DPM samplers can be summarized as follows:
\begin{align}
    \tilde{\bs x}_{t_i} = A_{t_{i-1}}^{t_i} \tilde{\bs x}_{t_{i-1}}^{\rm c} + \sum_{m=1}^{p-1} B_{t_{i-m}}^{t_{i}} \bs \beta_\theta(\tilde{\bs x}_{t_{i-m}}, t_{i-m}),\\ 
    \tilde{\bs x}_{t_i}^{\rm c} = C_{t_{i-1}}^{t_i} \tilde{\bs x}_{t_{i-1}}^{\rm c} + \sum_{m=0}^{p-1} D_{t_{i-m}}^{t_{i}} \bs \beta_\theta(\tilde{\bs x}_{t_{i-m}}, t_{i-m}),\label{proof:general_corrector}
\end{align}
where the corrector step~\eqref{proof:general_corrector} is optional and $\bs x_{t_i}^{\rm c}=\bs x_{t_i}$ if no corrector is used. We use $\bs \beta_\theta$ to represent different parameterizations during the sampling, such as the noise-prediction $\bs\epsilon_\theta$~\cite{lu2022dpmsolver,zhang2022fast_deis}, data-prediction $\bs x_\theta$~\cite{lu2022dpmsolverpp,zhao2023unipc}, $v$-prediction $\bs v_\theta$~\cite{salimans2022progressive}, or the learned parameterization~\cite{zheng2023dpmv3}. The coefficients ($A,B,C,D$) are determined by the specific sampler and differ across the sampling steps.

\section{Convergence of DC-Solver}\label{app:convergence}

In this section, we shall show that if the original sampler has the convergence order $p+1$ under mild conditions, then the same order of convergence is maintained when combined with our Dynamic Compensation. We will prove for both predictor-only samplers~\cite{lu2022dpmsolverpp,song2020denoising_ddim} and predictor-corrector samplers~\cite{zhao2023unipc}. For the sake of simplicity, we use the $\ell-2$ norm by default to study the convergence.

\subsection{Assumptions}\label{app:assumptions}
We introduce some assumptions for the convenience of subsequent proofs. These assumptions are either common in ODE analysis or easy to satisfy.
\begin{assumption}\label{app:Lip}
The prediction model ${\bs \beta}_{\theta}(x,t)$ is Lipschitz continuous w.r.t. $x$.
\end{assumption}


\begin{assumption}\label{app:huniform}
    $h=\max_{1\leq i \leq M} h_i= \bigO (1/M)$, where $h_i$ denotes the sampling step size, and $M$ is the total number of sampling steps.
\end{assumption}


\begin{assumption} \label{app:coefforder}
    The coefficients in \eqref{proof:general_corrector} satisfy that $0 < C_1 \leq \|A_{t_{i-1}}^{t_i}\|_2 \leq C_2 $, $0<C_3h \leq \|B_{t_{i-m}}^{t_i}\|_2 \leq C_4h$, $0 < C_5 \leq \|C_{t_{i-1}}^{t_i}\|_2 \leq C_6 $ and $0<C_7h \leq \|D_{t_{i-m}}^{t_i}\|_2 \leq C_8h$ for sufficiently small $h$.
\end{assumption}
Assumption \ref{app:Lip}  is common in the analysis of ODEs.  Assumption \ref{app:huniform} assures that the step size is basically uniform.

Assumption \ref{app:coefforder} can be easily verified by the formulation of the samplers. For example, in data-prediction mode of UniPC~\cite{zhao2023unipc}, we have $A_{t_{i-1}}^{t_i}=\alpha_{t_i}/\alpha_{t_{i-1}}$, which are constants independent of $h_i$. Note that $B_{t_{i-1}}^{t_i} = \sigma_{t_i}(e^{h_i}-1)\left[ \sum_{m=1}^{p}\frac{a_m}{r_m}-1 \right]$ and $B_{t_{i-m}}^{t_i}=-\sigma_{t_i}(e^{h_i}-1)\frac{a_m}{r_m},m\neq1$, where $a_m,r_m\in \mathcal{O}(1)$, we have $B_{t_i-m}^{t_i}=\mathcal{O}(h)$. For $C_{t_{i-1}}^{t_i}$ and $D_{t_{i-m}}^{t_i}$, we can analogically derive the bound for the two coefficients. By examining the analytical form of other existing solvers~\cite{song2020denoising_ddim,lu2022dpmsolver,lu2022dpmsolverpp,zhang2022fast_deis,liu2022pseudo,zhao2023unipc}, we can similarly find that~\Cref{app:coefforder} always holds.


\subsection{Local Convergence}
\begin{theorem}\label{app:thmlocal}
    For any DPM sampler of $p+1$-th order of accuracy, \ie, $\mathbb{E}\|\tilde{\bs x}_{t_{i+1}}^{\rm c} - \tilde{\bs x}_{t_{i+1}}\|_2 \le C h_i^{p+2}$, applying dynamic compensation with the ratio $\rho_i^*$ will reduce the local truncation error and remain the $p+1$-th order of accuracy.
\end{theorem}
\begin{proof} Denote $\tilde{\bs x}_{t_{i+1}}^{{\rm c}, \rho_i}$ as the intermediate result at the next sampling step by using dynamic compensation ratio $\rho_i$. Observe that $\rho_i=1.0$ is equivalent to the original updating formula without the dynamic compensation, we have
\begin{align}
    \mathbb{E}\|\tilde{\bs x}_{t_{i+1}}^{{\rm c}, \rho_i^*} - \tilde{\bs x}_{t_{i+1}}\|_2 &\le  \mathbb{E}\|\tilde{\bs x}_{t_{i+1}}^{{\rm c}, 1.0} - \tilde{\bs x}_{t_{i+1}}\|_2\nonumber\\ &=\mathbb{E}\|\tilde{\bs x}_{t_{i+1}}^{\rm c} - \tilde{\bs x}_{t_{i+1}}\|_2 \le C h_i^{p+2}.
\end{align}
Therefore, the local truncation error is reduced and the order of accuracy after the DC is still $p+1$. 
\end{proof}
Note that the proof does not assume the detailed implementation of the sampler, indicating that the~\Cref{app:thmlocal} holds for both predictor-only samplers and predictor-corrector samplers.

\subsection{Global Convergence}
We first investigate the global convergence of Dynamic Compensation with a $p$-th order predictor-only sampler.
\begin{corollary}\label{app:coroPredictor}
    Assume that we have ${\{{\tilde{\bs x}}_{t_{i-k}}\}}_{k=1}^{p-1}$ and ${\{\bs \beta_{\theta}^{{\rho}_{i-k}^*}({\tilde{\bs x}}_{t_{i-k}},t_{i-k})\}}_{k=2}^{p-1}$ (denoted as ${\{\bs \beta_{\theta}^{{\rho}_{i-k}^*}\}}_{k=2}^{p-1}$) satisfying $\mathbb{E}\| {\tilde{\bs x}}_{t_{i-k}} - \bs x_{t_{i-k}} \|_2=\bigO(h^{p}),1\leq k \leq p-1$, and
$\mathbb{E}\| {\bs \beta}_{\theta}^{{\rho}_{i-k}^*} - {\bs \beta}_{\theta}(\bs x_{t_{i-k}},t_{i-k}) \|_2=\bigO(h^{p-1}),2\leq k \leq p-1$. If we use Predictor-$p$ together with Dynamic Compensation to estimate $\bs x_{t_i}$, we shall get ${\bs \beta}_{\theta}^{{\rho}_{i-1}^*} $ and ${\tilde{\bs x}}_{t_i}$ that satisfy
$\mathbb{E}\| {\bs \beta}_{\theta}^{{\rho}_{i-1}^*} - {\bs \beta}_{\theta}(\bs x_{t_{i-1}},t_{i-1}) \|_2 =\bigO(h^{p-1})$ and
$\mathbb{E}\| {\tilde{\bs x}}_{t_{i}} - \bs x_{t_{i}} \|_2=\bigO(h^{p})$.
\end{corollary}
\begin{proof}
It is obvious that for sufficiently large constants $C_{\bs \beta},C_{\bs x}$, we have
\begin{equation}\label{app:coro1Cbeta}
    \mathbb{E}\| {\bs \beta}_{\theta}^{{\rho}_{i-k}^*} - {\bs \beta}_{\theta}(\bs x_{t_{i-k}},t_{i-k}) \|_2\leq C_{\bs \beta}h^{p-1},2\leq k \leq p-1
\end{equation}
\begin{equation}\label{app:coro1Cx}
    \mathbb{E}\| {\tilde{\bs x}}_{t_{i-k}} - \bs x_{t_{i-k}} \|_2\leq C_xh^{p},1 \leq k \leq p-1
\end{equation}
When computer $\bs x_{t_i}$, we consider 3 different methods in this step. Firstly, if we continue to use Dynamic Compensation, we have
\begin{equation}\label{app:inductiondc}
    {\tilde{\bs x}}_{t_i}=A_{t_{i-1}}^{t_i}{\tilde{\bs x}}_{t_{i-1}} +\sum_{m=1}^{p-1} B_{t_{i-m}}^{t_i} {\bs \beta}_{\theta}^{\rho_{i-m}^*}.
\end{equation}
Otherwise, if we use the standard Predictor-$p$ at this step (which means to do not replace the ${\bs \beta}_{\theta}({\tilde{\bs x}}_{t_{i-1}},t_{i-1})$ with $\bs\beta_\theta^{\rho_{i-m}^*}$), we have the following result:
\begin{equation}\label{app:inductionpre}
    {\tilde{\bs x}}_{t_i}^{\rm p}=A_{t_{i-1}}^{t_i}{\tilde{\bs x}}_{t_{i-1}} +\sum_{m=2}^{p-1} B_{t_{i-m}}^{t_i} {\bs \beta}_{\theta}^{\rho_{i-m}^*} + B_{t_{i-1}}^{t_i} {\bs \beta}_{\theta}({\tilde{\bs x}}_{t_{i-1}},t_{i-1}).
\end{equation}
In the third case, we adopt the Predictor-$p$ to previous points on the ground truth trajectory:
\begin{equation}\label{app:inductiongt}
    {\bar{\bs x}}_{t_i}=A_{t_{i-1}}^{t_i}\bs x_{t_{i-1}} +\sum_{m=1}^{p-1} B_{t_{i-m}}^{t_i} {\bs \beta}_{\theta}(\bs x_{t_{i-m}},t_{i-m})
\end{equation}
Due to the $p$-th order of accuarcy of Predictor-$p$, we have
\begin{equation}\label{app:predictorgtorder}
\mathbb{E} \| {\bar{\bs x}}_{t_i} -\bs x_{t_i}\|_2 = \bigO(h^{p+1})
\end{equation}
Comparing \eqref{app:inductiongt} and \eqref{app:inductionpre}, we obtain
\begin{equation}
\begin{split}
       & {\tilde{\bs x}}_{t_i}^{\rm p} - {\bar{\bs x}}_{t_i} = A_{t_{i-1}}^{t_i}({\tilde{\bs x}}_{t_{i-1}}-\bs x_{t_{i-1}}) \\
       & + \sum_{m=2}^{p-1} B_{t_{i-m}}^{t_i} \left[ {\bs \beta}_{\theta}^{\rho_{i-m}^*}-{\bs \beta}_{\theta}(\bs x_{t_{i-m}},t_{i-m}) \right] \\
       & + B_{t_{i-1}}^{t_i} \left[ {\bs \beta}_{\theta}({\tilde{\bs x}}_{t_{i-1}},t_{i-1})-{\bs \beta}_{\theta}(\bs x_{t_{i-1}},t_{i-1}) \right]
\end{split}
\end{equation}
Under Assumption \ref{app:Lip}, Assumption \ref{app:coefforder}, \eqref{app:coro1Cbeta} and \eqref{app:coro1Cx}, it follows that,
\begin{equation}\label{app:pp_gtp_order}
\begin{split}
    &\mathbb{E} \|{\tilde{\bs x}}_{t_i}^{\rm p} - {\bar{\bs x}}_{t_i} \|_2 \leq C_2C_xh^{p} \\
    & + \sum_{m=2}^{p-1} C_4 C_{\bs \beta}h^{p}  +  C_4LC_x h^{p+1} =\bigO(h^{p})
\end{split}
\end{equation}
By \eqref{app:predictorgtorder} and \eqref{app:pp_gtp_order}, we have
\begin{equation}
    \mathbb{E} \|  {\tilde{\bs x}}_{t_i}^{\rm p}-\bs x_{t_i} \|_2  = \bigO(h^{p})
\end{equation}
Observing that DC-Solver-$p$ is equivalent to Predictor-$p$ when $\rho_{i-1}=1.0$, we have
\begin{equation}
    \mathbb{E} \|  {\tilde{\bs x}}_{t_i}-\bs x_{t_i} \|_2 \leq  \mathbb{E} \|  {\tilde{\bs x}}_{t_i}^{\rm p}-\bs x_{t_i} \|_2 = \bigO(h^{p}).
\end{equation}
Combining with \eqref{app:predictorgtorder}, we get
\begin{equation}\label{app:predictorpreorder}
    \mathbb{E} \|  {\tilde{\bs x}}_{t_i}-{\bar{\bs x}}_{t_i} \|_2 = \bigO(h^{p}) \leq C_{9}h^{p}
\end{equation}
Subtracting \eqref{app:inductiongt} from \eqref{app:inductiondc}, we have
\begin{equation}
\begin{split}
       & {\tilde{\bs x}}_{t_i} - {\bar{\bs x}}_{t_i} = A_{t_{i-1}}^{t_i}({\tilde{\bs x}}_{t_{i-1}}-\bs x_{t_{i-1}}) \\
       & + \sum_{m=2}^{p-1} B_{t_{i-m}}^{t_i} \left[ {\bs \beta}_{\theta}^{\rho_{i-m}^*}-{\bs \beta}_{\theta}(\bs x_{t_{i-m}},t_{i-m}) \right] \\
       & + B_{t_{i-1}}^{t_i} \left[ {\bs \beta}_{\theta}^{\rho_{i-1}^*}-{\bs \beta}_{\theta}(\bs x_{t_{i-1}},t_{i-1}) \right]
\end{split}
\end{equation}
Thus, given \eqref{app:predictorpreorder}, \eqref{app:coro1Cbeta}, \eqref{app:coro1Cx}, we obtain
\begin{equation}
\begin{split}
      & \mathbb{E} \left\|B_{t_{i-1}}^{t_i} \left[ {\bs \beta}_{\theta}^{\rho_{i-1}^*}-{\bs \beta}_{\theta}(\bs x_{t_{i-1}},t_{i-1}) \right]\right\|_2 \\
      &= \left\|{\tilde{\bs x}}_{t_i} - {\bar{\bs x}}_{t_i} -A_{t_{i-1}}^{t_i}({\tilde{\bs x}}_{t_{i-1}}-\bs x_{t_{i-1}})\right. \\
      \quad & - \left.\sum_{m=2}^{p-1} B_{t_{i-m}}^{t_i} \left[ {\bs \beta}_{\theta}^{\rho_{i-m}^*}-{\bs \beta}_{\theta}(\bs x_{t_{i-m}},t_{i-m}) \right]\right\|_2 \\
      & \leq C_9h^p + C_2C_xh^p + \sum_{m=2}^{p-1}C_4C_{\bs \beta}h^{p} \\
      & = \bigO(h^{p})
\end{split}
\end{equation}
Note that $\|B_{t_{i-1}}^{t_i}\|_2 \geq C_3 h $ according to Assumption \ref{app:coefforder}, we have
\begin{equation}\label{app:predictorbetaorder}
    \mathbb{E} \| {\bs \beta}_{\theta}^{\rho_{i-1}^*}-{\bs \beta}_{\theta}(\bs x_{t_{i-1}},t_{i-1}) \|_2 = \bigO(h^{p-1}).
\end{equation}
Above all, \eqref{app:predictorpreorder} and \eqref{app:predictorbetaorder} establish the correctness of the corollary.

\end{proof}

\begin{theorem}
    For any predictor-only sampler of $p$-th order of convergence, applying Dynamic Compensation with ratio $\rho_i^*$ will maintain the $p$-th order of convergence.
\end{theorem}
\begin{proof}
We will use mathematical induction to prove it. Denote ${\{\bs \beta_{\theta}^{{\rho}_{k}^*}\}}_{k=0}^{i-1}={\{\bs \beta_{\theta}^{{\rho}_{k}^*}({\tilde{\bs x}}_{t_{k}},t_{k})\}}_{k=0}^{i-1}$, we define $P_i$ as the proposition that $\mathbb{E}\| {\bs \beta}_{\theta}^{{\rho}_{k}^*} - {\bs \beta}_{\theta}(\bs x_{t_{k}},t_{k}) \|_2=\bigO(h^{p-1}),0\leq k \leq i-1$, and $\mathbb{E}\| {\tilde{\bs x}}_{t_{k}} - \bs x_{t_{k}} \|_2=\bigO(h^{p}),0\leq k \leq i$.

In the first $K$ steps (namely the warm-up steps), we only use the Predictor-$p$ without the Dynamic Compensation. Since Predictor-p has $p$-th order of convergence, it's obvious that $\mathbb{E}\| {\tilde{\bs x}}_{t_{k}} - \bs x_{t_{k}} \|_2=\bigO(h^{p}),0\leq k \leq K$. 
Under Assumption \ref{app:Lip}, we also have
\begin{equation}
\begin{split}
   & \mathbb{E}\| {\bs \beta}_{\theta}^{{\rho}_{k}^*} - {\bs \beta}_{\theta}(\bs x_{t_{k}},t_{k}) \|_2 = \mathbb{E}\| {\bs \beta}_{\theta}({\tilde{\bs x}}_{t_{k}},t_{k}) - {\bs \beta}_{\theta}(\bs x_{t_{k}},t_{k}) \|_2 \\
   & \leq \mathbb{E}\|{\tilde{\bs x}}_{t_{k}} - \bs x_{t_{k}} \|_2 =\bigO(h^p) \leq \bigO(h^{p-1}),  \forall 0\le k\le K-1
\end{split}
\end{equation}
Thus, we show that $P_{K}$ is true. Recall the result in \Cref{app:coroPredictor}, we can then use mathematical induction to prove that $P_M$ is true, where $M$ is the NFE. This indicates that $\mathbb{E}\|{\tilde{\bs x}}_{t_M} -\bs x_{t_M}\|_2 = \bigO(h^{p})$,  which concludes the proof that the convergence order is still $p$ with the Dynamic Compensation
\end{proof}

We then provide the proof of the convergence order when applying Dynamic Compensation to predictor-corrector solvers.

\begin{corollary}\label{app:coroPreCor}
    Assume that we have ${\{{\tilde{\bs x}}_{t_{i-k}}^{\rm c}\}}_{k=1}^{p-1}$, ${\{{\tilde{\bs x}}_{t_{i-k}}\}}_{k=1}^{p-1}$, and ${\{\bs \beta_{\theta}^{{\rho}_{i-k}^*}({\tilde{\bs x}}_{t_{i-k}}^{\rm c},t_{i-k})\}}_{k=2}^{p-1}$ (denoted as ${\{\bs \beta_{\theta}^{{\rho}_{i-k}^*}\}}_{k=2}^{p-1}$), which satisfy 
    $\mathbb{E}\| {\bs \beta}_{\theta}^{{\rho}_{i-k}^*} - {\bs \beta}_{\theta}(\bs x_{t_{i-k}},t_{i-k}) \|_2=\bigO(h^{p}),2 \leq k \leq p-1$ , $\mathbb{E}\| {\tilde{\bs x}}_{t_{i-k}}^{\rm c} - \bs x_{t_{i-k}} \|_2=\bigO(h^{p+1}),1 \leq k \leq p-1$, and $\mathbb{E}\| {\tilde{\bs x}}_{t_{i-k}} - \bs x_{t_{i-k}} \|_2=\bigO(h^{p}),1 \leq k \leq p-1$. Then using Predictor-Corrector-$p$ combined with Dynamic Compensation to estimate $\bs x_{t_i}$, we can calculate ${\bs \beta}_{\theta}^{{\rho}_{i-1}^*},{\tilde{\bs x}}_{t_i}^{\rm c} ,{\tilde{\bs x}}_{t_i}$, that satisfy
$\mathbb{E}\| {\bs \beta}_{\theta}^{{\rho}_{i-1}^*} - {\bs \beta}_{\theta}(\bs x_{t_{i-1}},t_{i-1}) \|_2 =\bigO(h^{p})$, $\mathbb{E}\| {\tilde{\bs x}}_{t_{i}}^{\rm c} - \bs x_{t_{i}} \|_2=\bigO(h^{p+1})$ and
$\mathbb{E}\| {\tilde{\bs x}}_{t_{i}} - \bs x_{t_{i}} \|_2=\bigO(h^{p})$
\end{corollary}

\begin{proof}
It is obvious that, there exists sufficiently large constants $C_{\bs \beta},C_x,C_y$, such that
\begin{equation}\label{app:coro2Cbeta}
    \mathbb{E}\| {\bs \beta}_{\theta}^{{\rho}_{i-k}^*} - {\bs \beta}_{\theta}(\bs x_{t_{i-k}},t_{i-k}) \|_2\leq C_{\bs \beta}h^{p},2 \leq k \leq p-1
\end{equation}
\begin{equation}\label{app:coro2Cxc}
    \mathbb{E}\| {\tilde{\bs x}}_{t_{i-k}}^{\rm c} - \bs x_{t_{i-k}} \|_2 \leq C_xh^{p+1},1 \leq k \leq p-1
\end{equation}
\begin{equation}\label{app:coro2Cx}
    \mathbb{E}\| {\tilde{\bs x}}_{t_{i-k}} - \bs x_{t_{i-k}} \|_2\leq C_yh^{p},1 \leq k \leq p-1
\end{equation}
When estimating $\bs x_{t_i}$, we consider three different methods in this step. First, if we use Dynamic Compensation, we have
\begin{align}
    &{\tilde{\bs x}}_{t_i}=A_{t_{i-1}}^{t_i}{\tilde{\bs x}}_{t_{i-1}}^{\rm c} +\sum_{m=1}^{p-1} B_{t_{i-m}}^{t_i} {\bs \beta}_{\theta}^{\rho_{i-m}^*}\label{app:inductiondcp}\\
    &{\tilde{\bs x}}_{t_i}^{\rm c}=C_{t_{i-1}}^{t_i}{\tilde{\bs x}}_{t_{i-1}}^{\rm c} +\sum_{m=1}^{p-1} D_{t_{i-m}}^{t_i} {\bs \beta}_{\theta}^{\rho_{i-m}^*} + D_{t_i}^{t_i} {\bs \beta}_{\theta}({\tilde{\bs x}}_{t_i},t_i)\label{app:inductiondcc}
\end{align}
Otherwise, if we use the standard Predictor-Corrector-$p$ without DC at this step, we get
\begin{equation}\label{app:inductionpcp}
    {\bar{\bs x}}_{t_i}=A_{t_{i-1}}^{t_i}{\tilde{\bs x}}_{t_{i-1}}^{\rm c} +\sum_{m=2}^{p-1} B_{t_{i-m}}^{t_i} {\bs \beta}_{\theta}^{\rho_{i-m}^*} + B_{t_{i-1}}^{t_i} {\bs \beta}_{\theta}({\tilde{\bs x}}_{t_{i-1}},t_{i-1})
\end{equation}
\begin{equation}\label{app:inductionpcc}
\begin{split}
    & {\bar{\bs x}}_{t_i}^{\rm c}  = C_{t_{i-1}}^{t_i}{\tilde{\bs x}}_{t_{i-1}}^{\rm c} +\sum_{m=2}^{p-1} D_{t_{i-m}}^{t_i} {\bs \beta}_{\theta}^{\rho_{i-m}^*}+ D_{t_{i-1}}^{t_i} {\bs \beta}_{\theta}({\tilde{\bs x}}_{t_{i-1}},t_{i-1})  \\
    & + D_{t_i}^{t_i} {\bs \beta}_{\theta}({\bar{\bs x}}_{t_i},t_i)
\end{split}
\end{equation}
Finally, we use Predictor-Corrector-$p$ to previous points on the ground truth trajectory, we have:
\begin{equation}\label{app:inductiongtp}
    {\hat{\bs x}}_{t_i}=A_{t_{i-1}}^{t_i}\bs x_{t_{i-1}} +\sum_{m=1}^{p-1} B_{t_{i-m}}^{t_i} {\bs \beta}_{\theta}(\bs x_{t_{i-m}},t_{i-m})
\end{equation}
\begin{equation}\label{app:inductiongtc}
    {\hat{\bs x}}_{t_i}^{\rm c}=C_{t_{i-1}}^{t_i}\bs x_{t_{i-1}} +\sum_{m=1}^{p-1} D_{t_{i-m}}^{t_i} {\bs \beta}_{\theta}(\bs x_{t_{i-m}},t_{i-m})+D_{t_i}^{t_i}{\bs \beta}_{\theta}({\hat{\bs x}}_{t_i},t_i)
\end{equation}
Due to Predictor-Corrector-$p$'s $p+1$-th convergence order, we have
\begin{equation} \label{app:gtc_gtorder}
\mathbb{E} \| {\hat{\bs x}}_{t_i}^{\rm c} -\bs x_{t_i}\|_2 =\bigO(h^{p+2}) 
\end{equation}
Based on Assumption \ref{app:Lip} and \eqref{app:coro2Cx}, we also know that 
\begin{equation}\label{app:betap_betagt}
\begin{split}
   & \mathbb{E}\| {\bs \beta}_{\theta}({\tilde{\bs x}}_{t_{i-1}},t_{i-1}) - {\bs \beta}_{\theta}({{\bs x}}_{t_{i-1}},t_{i-1}) \|_2 \\
   & \leq L\mathbb{E}\| {\tilde{\bs x}}_{t_{i-1}} - {{\bs x}}_{t_{i-1}}  \|_2 = \bigO(h^p)
\end{split}
\end{equation}
Subtracting \eqref{app:inductiongtc} from \eqref{app:inductionpcc}, we obtain
\begin{equation}\label{app:barc_hatc}
\begin{split}
       & {\bar{\bs x}}_{t_i}^{\rm c} - {\hat{\bs x}}_{t_i}^{\rm c} = C_{t_{i-1}}^{t_i}({\tilde{\bs x}}_{t_{i-1}}^{\rm c}-\bs x_{t_{i-1}}) \\
       & + \sum_{m=2}^{p-1} D_{t_{i-m}}^{t_i} \left[ {\bs \beta}_{\theta}^{\rho_{i-m}^*}-{\bs \beta}_{\theta}(\bs x_{t_{i-m}},t_{i-m}) \right] \\
       & + D_{t_{i-1}}^{t_i} \left[ {\bs \beta}_{\theta}({\tilde{\bs x}}_{t_{i-1}},t_{i-1})-{\bs \beta}_{\theta}(\bs x_{t_{i-1}},t_{i-1}) \right] \\
       & + D_{t_{i}}^{t_i} \left[ {\bs \beta}_{\theta}({\bar{\bs x}}_{t_{i}},t_{i})-{\bs \beta}_{\theta}({\hat{\bs x}}_{t_{i}},t_{i}) \right]
\end{split}
\end{equation}
Under Assumption \ref{app:Lip}, Assumption \ref{app:coefforder}, \eqref{app:betap_betagt}, \eqref{app:coro2Cbeta}, \eqref{app:coro2Cxc} and \eqref{app:coro2Cx}, it follows that,
\begin{equation}\label{app:betapc_betagtorder}
\begin{split}
   & \mathbb{E}\| {\bs \beta}_{\theta}({\bar{\bs x}}_{t_i},t_i) - {\bs \beta}_{\theta}({\hat{\bs x}}_{t_i},t_i) \|_2  \leq L \mathbb{E}\| {\bar{\bs x}}_{t_i} - {\hat{\bs x}}_{t_i} \|_2 \\
   & = L\mathbb{E}\| A_{t_{i-1}}^{t_i} ({\tilde{\bs x}}_{t_{i-1}}^{\rm c}-\bs x_{t_{i-1}}) \\
   & + \sum_{m=2}^{p-1} B_{t_{i-m}}^{t_i} \left[ {\bs \beta}_{\theta}^{\rho_{i-m}^*} - {\bs \beta}_{\theta}(\bs x_{t_{i-m}},t_{i-m}) \right]\\
   & + B_{t_{i-1}}^{t_i}\left[  {\bs \beta}_{\theta}({\tilde{\bs x}}_{t_{i-1}},t_{i-1}) -{\bs \beta}_{\theta}(\bs x_{t_{i-1}},t_{i-1})  \right]   \|_2 \\
   & \leq L(C_2C_xh^{p+1}+\sum_{m=2}^{p-1}C_4C_{\bs \beta}h^{p+1}+C_4LC_yh^{p+1}) \\ 
   & = \bigO(h^{p+1}) \leq C_{10}h^{p+1}
\end{split}
\end{equation}
Therefore, according to Assumption \ref{app:coefforder}, \eqref{app:coro2Cbeta}, \eqref{app:coro2Cxc}, \eqref{app:coro2Cx},  \eqref{app:barc_hatc} and \eqref{app:betapc_betagtorder}, we get 
\begin{equation}
\begin{split}
    \mathbb{E} \|{\bar{\bs x}}_{t_i}^{\rm c} - {\hat{\bs x}}_{t_i}^{\rm c} \|_2 & \leq C_6C_xh^{p+1} + \sum_{m=2}^{p-1} C_8C_{\bs \beta}h^{p+1} \\
    & + C_8LC_yh^{p+1} + C_8C_{10}h^{p+2}  \\ 
    & = \bigO(h^{p+1})
\end{split}
\end{equation}
Given \eqref{app:gtc_gtorder}, we have
\begin{equation} \label{app:pcc_gtorder}
\mathbb{E} \| {\bar{\bs x}}_{t_i}^{\rm c} -\bs x_{t_i}\|_2 =\bigO(h^{p+1})  
\end{equation}
Observe that DC-Solver-$p$ is equivalent to Predictor-Corrector-$p$ when $\rho_{i-1}=1.0$, we have
\begin{equation}\label{app:coroxcorder}
    \mathbb{E} \|  {\tilde{\bs x}}_{t_i}^{\rm c}-\bs x_{t_i} \|_2 \leq  \mathbb{E} \|  {\bar{\bs x}}_{t_i}^{\rm c}-\bs x_{t_i} \|_2 = \bigO(h^{p+1})
\end{equation}
Combining with \eqref{app:pcc_gtorder}, we get
\begin{equation} \label{app:dcc_gtcorder}
    \mathbb{E} \|  {\tilde{\bs x}}_{t_i}^{\rm c}-{\bar{\bs x}}_{t_i}^{\rm c} \|_2 = \bigO(h^{p+1})
\end{equation}
Comparing \eqref{app:inductiondcc} and \eqref{app:inductionpcc}, we have
\begin{equation}\label{app:dcc_gtc}
\begin{split}
        {\tilde{\bs x}}_{t_i}^{\rm c} - {\bar{\bs x}}_{t_i}^{\rm c} & = D_{t_{i-1}}^{t_i}\left[ {\bs \beta}_{\theta}^{\rho_{i-1}^*} - {\bs \beta}_{\theta}({\tilde{\bs x}}_{t_{i-1}},t_{i-1}) \right] \\
       & + D_{t_i}^{t_i}\left[ {\bs \beta}_{\theta}({\tilde{\bs x}}_{t_i},t_i) - {\bs \beta}_{\theta}({\bar{\bs x}}_{t_i},t_i) \right]
\end{split}
\end{equation}
Under Assumption \ref{app:coefforder} and \ref{app:Lip}, concerning about the order of the coefficients, we can know that
\begin{equation}\label{app:comparecoefforder}
\begin{split}
   & \mathbb{E}\| D_{t_i}^{t_i}\left[ {\bs \beta}_{\theta}({\tilde{\bs x}}_{t_i},t_i) - {\bs \beta}_{\theta}({\bar{\bs x}}_{t_i},t_i) \right] \|_2 \\
   & \leq L\|D_{t_i}^{t_i}\|_2\| B_{t_{i-1}}^{t_i}\|_2\mathbb{E}\|  {\bs \beta}_{\theta}^{\rho_{i-1}^*} - {\bs \beta}_{\theta}({\tilde{\bs x}}_{t_{i-1}},t_{i-1}) \|_2 \\
   & \ll \mathbb{E}\|D_{t_{i-1}}^{t_i}\left[ {\bs \beta}_{\theta}^{\rho_{i-1}^*} - {\bs \beta}_{\theta}({\tilde{\bs x}}_{t_{i-1}},t_{i-1}) \right]\|_2 
\end{split} 
\end{equation}
Leveraging \eqref{app:dcc_gtcorder}, \eqref{app:dcc_gtc} with \eqref{app:comparecoefforder}, we have
\begin{equation}
       \mathbb{E}\|D_{t_{i-1}}^{t_i}\left[ {\bs \beta}_{\theta}^{\rho_{i-1}^*} - {\bs \beta}_{\theta}({\tilde{\bs x}}_{t_{i-1}},t_{i-1}) \right]\|_2  = \bigO(h^{p+1})
\end{equation}
Thus, considering that $\|D_{t_i}^{t_i}\|_2 \geq C_7h$ in Assumption \ref{app:coefforder}, we can get 
\begin{equation}\label{app:betarho_betap}
    \| {\bs \beta}_{\theta}^{\rho_{i-1}^*} - {\bs \beta}_{\theta}({\tilde{\bs x}}_{t_{i-1}},t_{i-1}) \|_2 = \bigO(h^{p})
\end{equation}
Given \eqref{app:betap_betagt} and \eqref{app:betarho_betap}, we further obtain
\begin{equation}\label{app:corobetaorder}
     \| {\bs \beta}_{\theta}^{\rho_{i-1}^*} - {\bs \beta}_{\theta}({{\bs x}}_{t_{i-1}},t_{i-1}) \|_2 = \bigO(h^{p}) \leq C_{11}h^p
\end{equation}
Subtracting \eqref{app:inductiongtp} from \eqref{app:inductiondcp}, we obtain
\begin{equation}
\begin{split}
    & \mathbb{E}\| {\tilde{\bs x}}_{t_i} - {\hat{\bs x}}_{t_i} \|_2 =\mathbb{E} \| A_{t_{i-1}}^{t_i}({\tilde{\bs x}}_{t_{i-1}}^{\rm c} - \bs x_{t_{i-1}}) \\
    & +  B_{t_{i-1}}^{t_i} \left[ {\bs \beta}_{\theta}^{\rho_{i-1}^*} - {\bs \beta}_{\theta}({{\bs x}}_{t_{i-1}},t_{i-1}) \right] \\
    & + \sum_{m=2}^{p-1} B_{t_{i-m}}^{t_i} \left[ {\bs \beta}_{\theta}^{\rho_{i-m}^*} - {\bs \beta}_{\theta}(\bs x_{t_{i-m}},t_{i-m}) \right]  \|_2 \\
    & \leq C_2C_xh^{p+1} + C_4C_{11}h^{p+1}+ \sum_{m=2}^{p-1}C_4C_{\bs \beta}h^{p+1} \\
    & \leq \bigO(h^{p})
\end{split}
\end{equation}
Since $\mathbb{E}\|{\hat{\bs x}}_{t_i}-\bs x_{t_i}\|_2=\bigO(h^{p+1})$, we have
\begin{equation}\label{app:coroxporder}
    \mathbb{E}\| {\tilde{\bs x}}_{t_i} - {{\bs x}}_{t_i} \|_2 \leq \bigO(h^p)
\end{equation}
Above all, \eqref{app:coroxcorder}, \eqref{app:corobetaorder} and \eqref{app:coroxporder} imply the validity of the corollary.
\end{proof}

\begin{theorem}
    For any predictor-corrector sampler of $(p+1)$-th order of convergence, applying dynamic compensation with ratio $\rho_i^*$ will remain the $(p+1)$-th order of convergence.
\end{theorem}
\begin{proof}
We use mathematical induction to proof this. Suppose we have ${\{{\tilde{\bs x}}_{t_{k}}^{\rm c}\}}_{k=0}^{i}$, ${\{{\tilde{\bs x}}_{t_{k}}\}}_{k=0}^{i}$ and ${\{\bs \beta_{\theta}^{{\rho}_{k}^*}({\tilde{\bs x}}_{t_{k}}^{\rm c},t_{k})\}}_{k=0}^{i-1}$ denoted as ${\{\bs \beta_{\theta}^{{\rho}_{k}^*}\}}_{k=0}^{i-1}$. First, we define $P_i$ as the proposition that $\mathbb{E}\| {\bs \beta}_{\theta}^{{\rho}_{k}^*} - {\bs \beta}_{\theta}(\bs x_{t_{k}},t_{k}) \|_2=\bigO(h^{p}),0 \leq k \leq i-1$ , $\mathbb{E}\| {\tilde{\bs x}}_{t_{k}}^{\rm c} - \bs x_{t_{k}} \|_2=\bigO(h^{p+1}),0\leq k \leq i$ and $\mathbb{E}\| {\tilde{\bs x}}_{t_{k}} - \bs x_{t_{k}} \|_2=\bigO(h^{p}), 0\leq k \leq i$.\\
In the first $K$ steps, we only use Predictor-Corrector-$p$ without the Dynamic Compensation. Since Predictor-Corrector-$p$ has $(p+1)$-th order of convergence, it's obvious that $\mathbb{E}\| {\tilde{\bs x}}_{t_{k}}^{\rm c} - \bs x_{t_{k}} \|_2=\bigO(h^{p+1}),0\leq k \leq K$, and $\mathbb{E}\| {\tilde{\bs x}}_{t_{k}} - \bs x_{t_{k}} \|_2=\bigO(h^{p}),0\leq k \leq K$.
Under Assumption \ref{app:Lip}, we also know, for $k\in [0, K-1]$,
\begin{equation}
\begin{split}
    \mathbb{E}\| {\bs \beta}_{\theta}^{{\rho}_{k}^*} - {\bs \beta}_{\theta}(\bs x_{t_{k}},t_{k}) \|_2 & = \mathbb{E}\| {\bs \beta}_{\theta}({\tilde{\bs x}}_{t_{k}},t_{k}) - {\bs \beta}_{\theta}(\bs x_{t_{k}},t_{k}) \|_2 \\
   & \leq L\mathbb{E}\|{\tilde{\bs x}}_{t_{k}} - \bs x_{t_{k}} \|_2 =\bigO(h^p)
\end{split}
\end{equation}
Thus, we show that $P_{K}$ is true. Similarly, using mathematical induction and the result in Corollary \ref{app:coroPreCor} we can know that $P_M$ is true, which implies that $\mathbb{E}\|{\tilde{\bs x}}_{t_M}^{\rm c} - \bs x_{t_M}\|_2 = \bigO(h^{p+1})$ and ends the proof. Therefore, we reach the conclusion that for a predictor-corrector sampler, the Dynamic Compensation will preserve the  $p+1$ convergence order.

\end{proof}

\begin{table}[!t]
\caption{\textbf{Detailed quantitative results on unconditional sampling.} We provide the comparisons of the FID$\downarrow$ of our DC-Solver and the previous method on FFHQ~\cite{karras2019ffhq}, LSUN-Church~\cite{yu2015lsun} and LSUN-Bedroom~\cite{yu2015lsun} with 5$\sim$10 NFE. We observe that our DC-Solver achieves the lowest FID on all three datasets.}\label{app:table_uncond}
\begin{subtable}{\linewidth}
\centering
\caption{FFHQ~\cite{karras2019ffhq}}
\adjustbox{width=.7\linewidth}{
\begin{tabular}{lcccccc}\toprule
\multicolumn{1}{l}{\multirow{2}[0]{*}{Method}} & \multicolumn{6}{c}{NFE} \\\cmidrule{2-7} & 5     & 6     & 7     & 8     & 9     & 10 \\\midrule
DPM-Solver++~\cite{lu2022dpmsolverpp} & 27.15 & 15.60 & 10.81 & 8.98 & 7.89 & 7.39\\
DEIS~\cite{zhang2022fast_deis} & 32.35 & 18.72 & 12.22 & 9.51 & 8.31 & 7.75\\
UniPC~\cite{zhao2023unipc} & 18.66 & 11.89 & 9.51 & 8.21 & 7.62 & 6.99\\
\rowcolor{Gray} DC-Solver (Ours) & \textbf{10.38} & \textbf{8.39} & \textbf{7.66} & \textbf{7.14} & \textbf{6.92} & \textbf{6.82}\\\bottomrule
\end{tabular}
}
\vspace{6pt}
\end{subtable}\\
\begin{subtable}{\linewidth}
\centering
\caption{LSUN-Church~\cite{yu2015lsun}}
\adjustbox{width=.7\linewidth}{
\begin{tabular}{lcccccc}\toprule
\multicolumn{1}{l}{\multirow{2}[0]{*}{Method}} & \multicolumn{6}{c}{NFE} \\\cmidrule{2-7} & 5     & 6     & 7     & 8     & 9     & 10 \\\midrule
DPM-Solver++~\cite{lu2022dpmsolverpp} & 17.57 & 9.71 & 6.45 & 4.97 & 4.25 & 3.87\\
DEIS~\cite{zhang2022fast_deis} & 15.01 & 8.45 & 5.71 & 4.49 & 3.86 & 3.57\\
UniPC~\cite{zhao2023unipc} & 11.98 & 6.90 & 5.08 & 4.28 & 3.86 & 3.61\\
\rowcolor{Gray} DC-Solver (Ours) & \textbf{7.47} & \textbf{4.70} & \textbf{3.91} & \textbf{3.46} & \textbf{3.23} & \textbf{3.06}\\\bottomrule
\end{tabular}
}
\vspace{6pt}
\end{subtable}\\
\begin{subtable}{\linewidth}
\centering
\caption{LSUN-Bedroom~\cite{yu2015lsun}}
\adjustbox{width=.7\linewidth}{
\begin{tabular}{lcccccc}\toprule
\multicolumn{1}{l}{\multirow{2}[0]{*}{Method}} & \multicolumn{6}{c}{NFE} \\\cmidrule{2-7} & 5     & 6     & 7     & 8     & 9     & 10 \\\midrule
DPM-Solver++~\cite{lu2022dpmsolverpp} & 18.13 & 8.33 & 5.15 & 4.14 & 3.77 & 3.61\\
DEIS~\cite{zhang2022fast_deis} & 16.68 & 8.75 & 6.13 & 5.11 & 4.66 & 4.41\\
UniPC~\cite{zhao2023unipc} & 12.14 & 6.13 & 4.53 & 4.05 & 3.81 & 3.64\\
\rowcolor{Gray} DC-Solver (Ours) & \textbf{7.40} & \textbf{5.29} & \textbf{4.27} & \textbf{3.98} & \textbf{3.74} & \textbf{3.52}\\\bottomrule
\end{tabular}
}
\end{subtable}
\end{table}

\section{More Analyses}

\begin{table*}[!t]
\caption{\textbf{Detailed quantitative results on conditional sampling.} We provide the comparisons between our DC-Solver and the previous method on Stable-Diffusion-1.5~\cite{rombach2022high} with different classifier-free guidance scale (CFG) and $\NFE\in[5, 10]$. The sampling quality is measured by the MSE$\downarrow$ between the generated latents and the ground truth latents (obtained by a 999-step DDIM). We demonstrate that DC-Solver consistently achieves the best result for different sampling configurations.}\label{app:table_cond}
\centering
\begin{subtable}{.49\textwidth}
\caption{$\CFG=1.0$}
\adjustbox{width=\textwidth}{
\begin{tabular}{lcccccc}\toprule
\multicolumn{1}{l}{\multirow{2}[0]{*}{Method}} & \multicolumn{6}{c}{NFE} \\\cmidrule{2-7} & 5     & 6     & 7     & 8     & 9     & 10 \\\midrule
DPM-Solver++~\cite{lu2022dpmsolverpp} & 0.277 & 0.232 & 0.204 & 0.188 & 0.177 & 0.169\\
DEIS~\cite{zhang2022fast_deis} & 0.299 & 0.252 & 0.223 & 0.203 & 0.191 & 0.181\\
UniPC~\cite{zhao2023unipc} & 0.245 & 0.206 & 0.184 & 0.172 & 0.166 & 0.161\\
\rowcolor{Gray} DC-Solver (Ours) & \textbf{0.176} & \textbf{0.163} & \textbf{0.150} & \textbf{0.150} & \textbf{0.147} & \textbf{0.144}\\\bottomrule
\end{tabular}
}
\end{subtable}\hfill
\begin{subtable}{.49\textwidth}
\caption{$\CFG=1.5$}
\adjustbox{width=\textwidth}{
\begin{tabular}{lcccccc}\toprule
\multicolumn{1}{l}{\multirow{2}[0]{*}{Method}} & \multicolumn{6}{c}{NFE} \\\cmidrule{2-7} & 5     & 6     & 7     & 8     & 9     & 10 \\\midrule
DPM-Solver++~\cite{lu2022dpmsolverpp} & 0.288 & 0.242 & 0.213 & 0.195 & 0.182 & 0.173\\
DEIS~\cite{zhang2022fast_deis} & 0.307 & 0.260 & 0.229 & 0.209 & 0.194 & 0.184\\
UniPC~\cite{zhao2023unipc} & 0.260 & 0.219 & 0.194 & 0.180 & 0.170 & 0.163\\
\rowcolor{Gray} DC-Solver (Ours) & \textbf{0.213} & \textbf{0.188} & \textbf{0.169} & \textbf{0.158} & \textbf{0.153} & \textbf{0.149}\\\bottomrule
\end{tabular}
}
\end{subtable}\\
\vspace{6pt}
\begin{subtable}{.49\textwidth}
\caption{$\CFG=2.5$}
\adjustbox{width=\textwidth}{
\begin{tabular}{lcccccc}\toprule
\multicolumn{1}{l}{\multirow{2}[0]{*}{Method}} & \multicolumn{6}{c}{NFE} \\\cmidrule{2-7} & 5     & 6     & 7     & 8     & 9     & 10 \\\midrule
DPM-Solver++~\cite{lu2022dpmsolverpp} & 0.339 & 0.293 & 0.262 & 0.239 & 0.221 & 0.208\\
DEIS~\cite{zhang2022fast_deis} & 0.354 & 0.307 & 0.274 & 0.250 & 0.231 & 0.217\\
UniPC~\cite{zhao2023unipc} & 0.321 & 0.277 & 0.247 & 0.226 & 0.208 & 0.195\\
\rowcolor{Gray} DC-Solver (Ours) & \textbf{0.293} & \textbf{0.257} & \textbf{0.231} & \textbf{0.212} & \textbf{0.194} & \textbf{0.186}\\\bottomrule
\end{tabular}
}
\end{subtable}\hfill
\begin{subtable}{.49\textwidth}
\caption{$\CFG=3.5$}
\adjustbox{width=\textwidth}{
\begin{tabular}{lcccccc}\toprule
\multicolumn{1}{l}{\multirow{2}[0]{*}{Method}} & \multicolumn{6}{c}{NFE} \\\cmidrule{2-7} & 5     & 6     & 7     & 8     & 9     & 10 \\\midrule
DPM-Solver++~\cite{lu2022dpmsolverpp} & 0.409 & 0.360 & 0.323 & 0.295 & 0.272 & 0.255\\
DEIS~\cite{zhang2022fast_deis} & 0.419 & 0.369 & 0.332 & 0.303 & 0.280 & 0.262\\
UniPC~\cite{zhao2023unipc} & 0.397 & 0.349 & 0.312 & 0.285 & 0.262 & 0.245\\
\rowcolor{Gray} DC-Solver (Ours) & \textbf{0.375} & \textbf{0.331} & \textbf{0.299} & \textbf{0.270} & \textbf{0.251} & \textbf{0.239}\\\bottomrule
\end{tabular}
}
\end{subtable}\\
\vspace{6pt}
\begin{subtable}{.49\textwidth}
\caption{$\CFG=4.5$}
\adjustbox{width=\textwidth}{
\begin{tabular}{lcccccc}\toprule
\multicolumn{1}{l}{\multirow{2}[0]{*}{Method}} & \multicolumn{6}{c}{NFE} \\\cmidrule{2-7} & 5     & 6     & 7     & 8     & 9     & 10 \\\midrule
DPM-Solver++~\cite{lu2022dpmsolverpp} & 0.490 & 0.437 & 0.392 & 0.358 & 0.330 & 0.308\\
DEIS~\cite{zhang2022fast_deis} & 0.496 & 0.441 & 0.397 & 0.364 & 0.336 & 0.314\\
UniPC~\cite{zhao2023unipc} & 0.483 & 0.430 & 0.386 & 0.352 & 0.324 & 0.302\\
\rowcolor{Gray} DC-Solver (Ours) & \textbf{0.461} & \textbf{0.412} & \textbf{0.369} & \textbf{0.337} & \textbf{0.314} & \textbf{0.291}\\\bottomrule
\end{tabular}
}
\end{subtable}\hfill
\begin{subtable}{.49\textwidth}
\caption{$\CFG=5.5$}
\adjustbox{width=\textwidth}{
\begin{tabular}{lcccccc}\toprule
\multicolumn{1}{l}{\multirow{2}[0]{*}{Method}} & \multicolumn{6}{c}{NFE} \\\cmidrule{2-7} & 5     & 6     & 7     & 8     & 9     & 10 \\\midrule
DPM-Solver++~\cite{lu2022dpmsolverpp} & 0.580 & 0.517 & 0.468 & 0.427 & 0.395 & 0.368\\
DEIS~\cite{zhang2022fast_deis} & 0.581 & 0.519 & 0.469 & 0.430 & 0.398 & 0.372\\
UniPC~\cite{zhao2023unipc} & 0.577 & 0.516 & 0.468 & 0.428 & 0.395 & 0.367\\
\rowcolor{Gray} DC-Solver (Ours) & \textbf{0.551} & \textbf{0.492} & \textbf{0.446} & \textbf{0.406} & \textbf{0.381} & \textbf{0.355}\\\bottomrule
\end{tabular}
}
\end{subtable}\\
\vspace{6pt}

\begin{subtable}{.49\textwidth}
\caption{$\CFG=6.5$}
\adjustbox{width=\textwidth}{
\begin{tabular}{lcccccc}\toprule
\multicolumn{1}{l}{\multirow{2}[0]{*}{Method}} & \multicolumn{6}{c}{NFE} \\\cmidrule{2-7} & 5     & 6     & 7     & 8     & 9     & 10 \\\midrule
DPM-Solver++~\cite{lu2022dpmsolverpp} & 0.687 & 0.612 & 0.556 & 0.512 & 0.474 & 0.441\\
DEIS~\cite{zhang2022fast_deis} & 0.684 & 0.610 & 0.554 & 0.511 & 0.474 & 0.442\\
UniPC~\cite{zhao2023unipc} & 0.691 & 0.618 & 0.563 & 0.517 & 0.479 & 0.445\\
\rowcolor{Gray} DC-Solver (Ours) & \textbf{0.654} & \textbf{0.587} & \textbf{0.531} & \textbf{0.488} & \textbf{0.457} & \textbf{0.426}\\\bottomrule
\end{tabular}
}
\end{subtable}\hfill
\begin{subtable}{.49\textwidth}
\caption{$\CFG=7.5$}
\adjustbox{width=\textwidth}{
\begin{tabular}{lcccccc}\toprule
\multicolumn{1}{l}{\multirow{2}[0]{*}{Method}} & \multicolumn{6}{c}{NFE} \\\cmidrule{2-7} & 5     & 6     & 7     & 8     & 9     & 10 \\\midrule
DPM-Solver++~\cite{lu2022dpmsolverpp} & 0.812 & 0.719 & 0.648 & 0.597 & 0.554 & 0.518\\
DEIS~\cite{zhang2022fast_deis} & 0.802 & 0.712 & 0.643 & 0.592 & 0.552 & 0.517\\
UniPC~\cite{zhao2023unipc} & 0.825 & 0.733 & 0.666 & 0.612 & 0.570 & 0.530\\
\rowcolor{Gray} DC-Solver (Ours) & \textbf{0.766} & \textbf{0.689} & \textbf{0.620} & \textbf{0.573} & \textbf{0.537} & \textbf{0.501}\\\bottomrule
\end{tabular}
}
\end{subtable}
\end{table*}

\subsection{Quantitative Results}
We now provide detailed quantitative results on both unconditional sampling and conditional sampling. For unconditional sampling, we list the numerical results on FFHQ~\cite{karras2019ffhq}, LSUN-Church~\cite{yu2015lsun} and LSUN-Bedroom~\cite{yu2015lsun} in~\Cref{app:table_uncond}. All the pre-trained DPMs are from Latent-Diffusion~\cite{rombach2022high} and we use  FID$\downarrow$ as the evaluation metric. We demonstrate that our DC-Solver consistently attains the lowest FID on all three datasets. For conditional sampling, we summarize the results in~\Cref{app:table_cond}, where we compare the sampling quality of different methods on various configurations of classifier-free guidance scale (CFG). Our results indicate that DC-Solver can outperform previous methods by large margins with different choices of CFG and NFE. 

\begin{figure*}[!t]\label{fig:viz:uncond}
    \centering
     \begin{adjustbox}{width=.72\textwidth}
    \begin{tabu}to .75\textwidth{*{4}{X[C]}}
    \multicolumn{4}{c}{\makecell[c]{FFHQ}}\\\toprule
      DPM++~\cite{lu2022dpmsolverpp} & DEIS~\cite{zhang2022fast_deis} & UniPC~\cite{zhao2023unipc} & \textbf{DC-Solver} \\\midrule
    \raisebox{-.5\height}{\includegraphics[width=.18\textwidth]{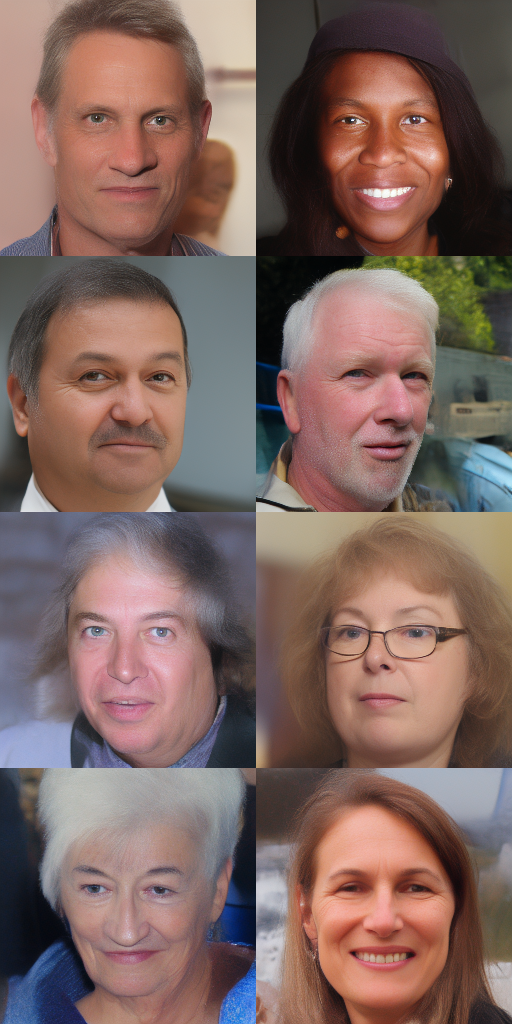}}
    & \raisebox{-.5\height}{\includegraphics[width=.18\textwidth]{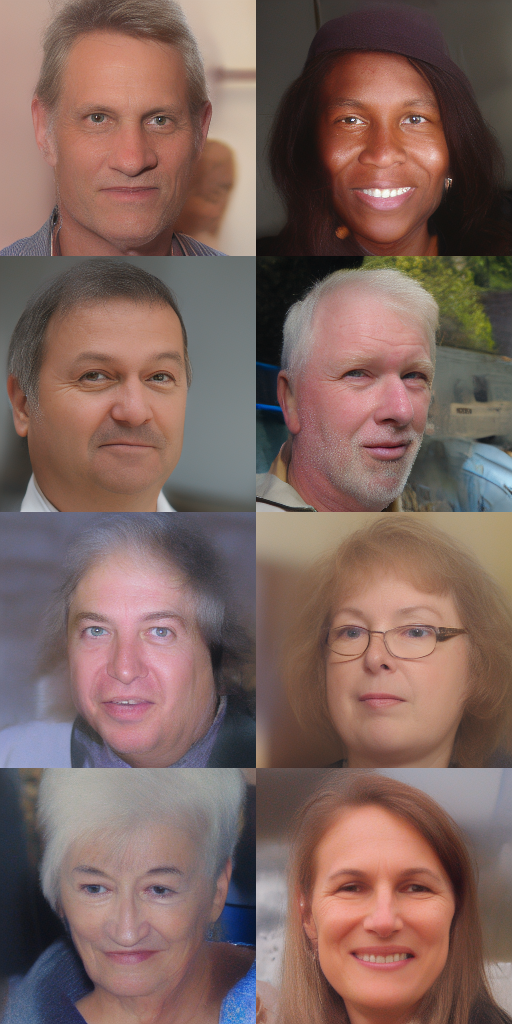}}
    & \raisebox{-.5\height}{\includegraphics[width=.18\textwidth]{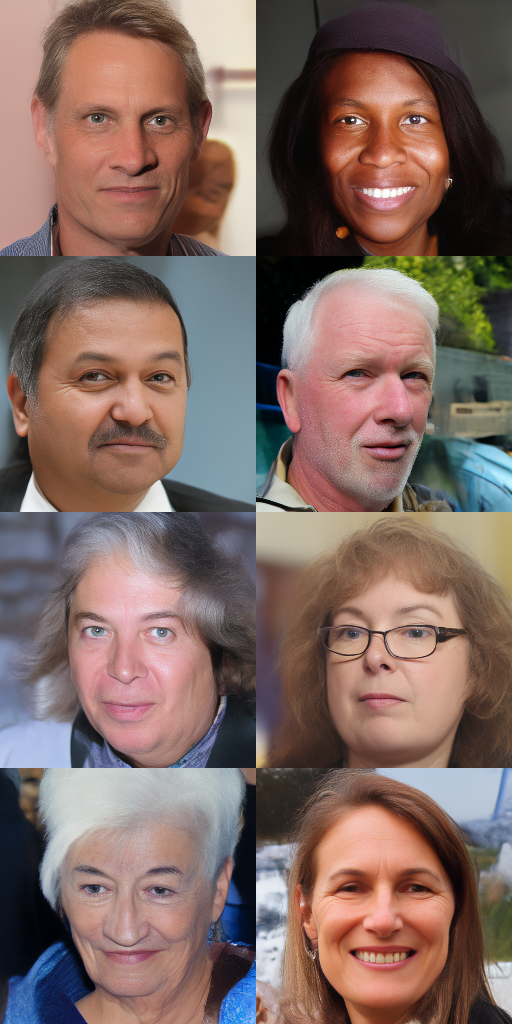}}
    & \raisebox{-.5\height}{\includegraphics[width=.18\textwidth]{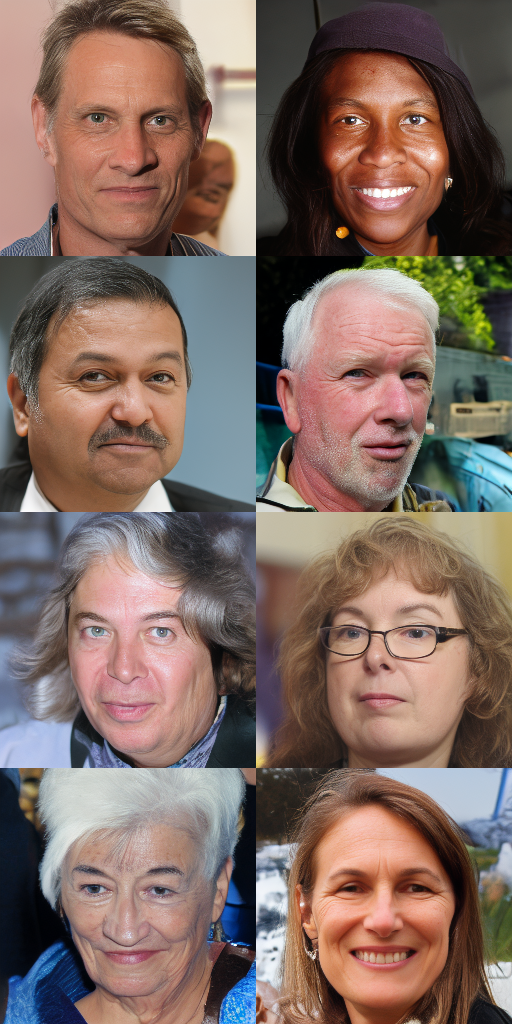}} \\\bottomrule
    \end{tabu}
    \end{adjustbox}

     \begin{adjustbox}{width=.72\textwidth}
    \begin{tabu}to.75\textwidth{*{4}{X[C]}}
    \multicolumn{4}{c}{\makecell[c]{LSUN Bedroom}}\\\toprule
      DPM++~\cite{lu2022dpmsolverpp} & DEIS~\cite{zhang2022fast_deis} & UniPC~\cite{zhao2023unipc} & \textbf{DC-Solver}  \\\midrule
    \raisebox{-.5\height}{\includegraphics[width=.18\textwidth]{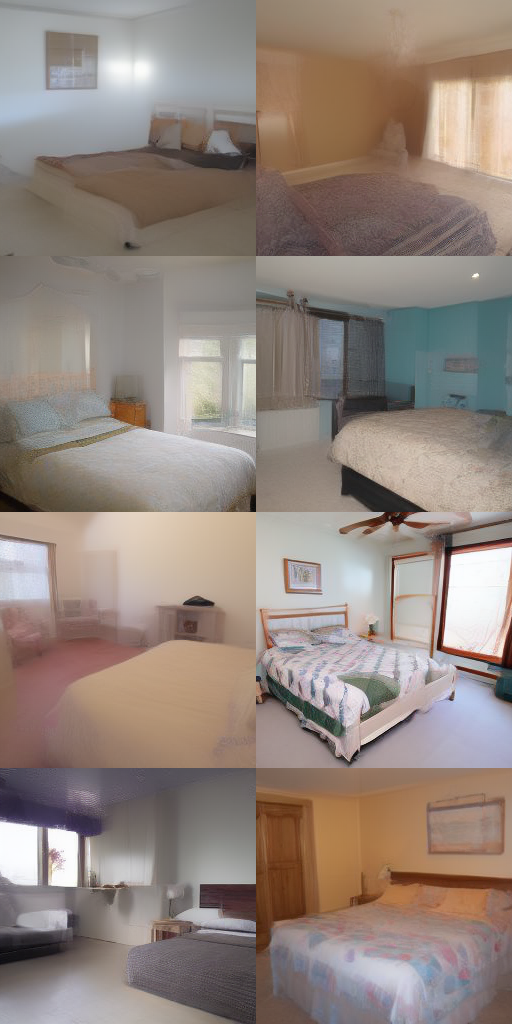}}
    & \raisebox{-.5\height}{\includegraphics[width=.18\textwidth]{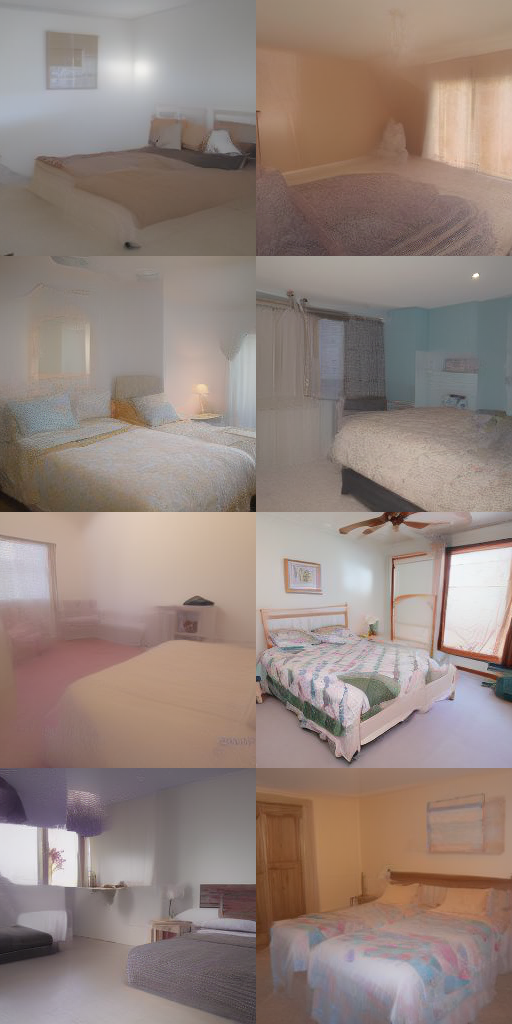}}
    & \raisebox{-.5\height}{\includegraphics[width=.18\textwidth]{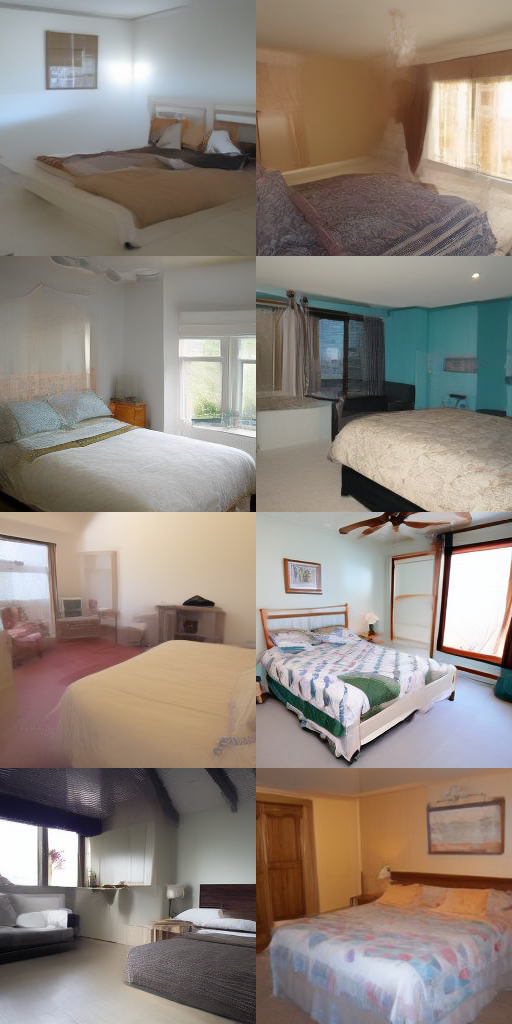}}
    & \raisebox{-.5\height}{\includegraphics[width=.18\textwidth]{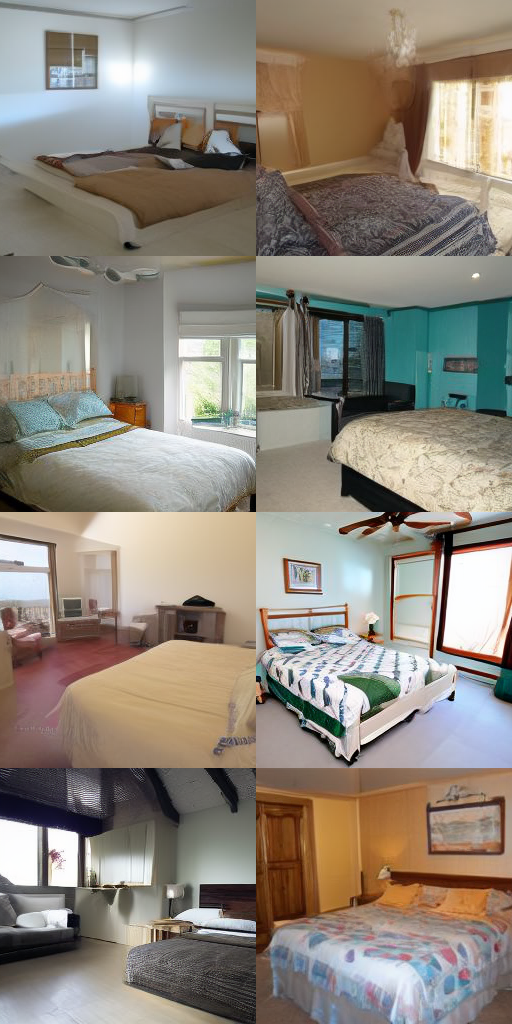}} \\\bottomrule
    \end{tabu}
    \end{adjustbox}
    
     \begin{adjustbox}{width=.72\textwidth}
    \begin{tabu}to.75\textwidth{*{4}{X[C]}}
    \multicolumn{4}{c}{\makecell[c]{LSUN Church}}\\\toprule
      DPM++~\cite{lu2022dpmsolverpp} & DEIS~\cite{zhang2022fast_deis} & UniPC~\cite{zhao2023unipc} & \textbf{DC-Solver} \\\midrule
    \raisebox{-.5\height}{\includegraphics[width=.18\textwidth]{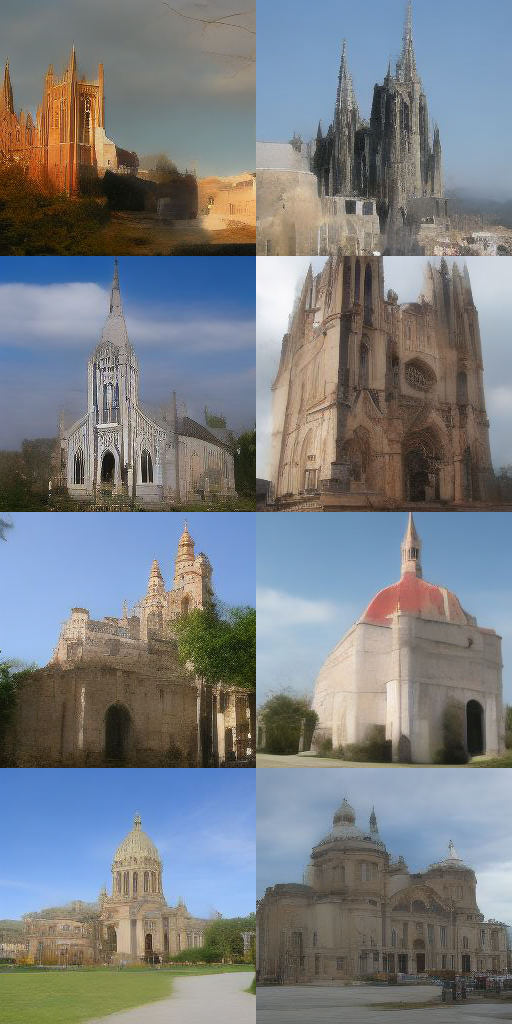}}
    & \raisebox{-.5\height}{\includegraphics[width=.18\textwidth]{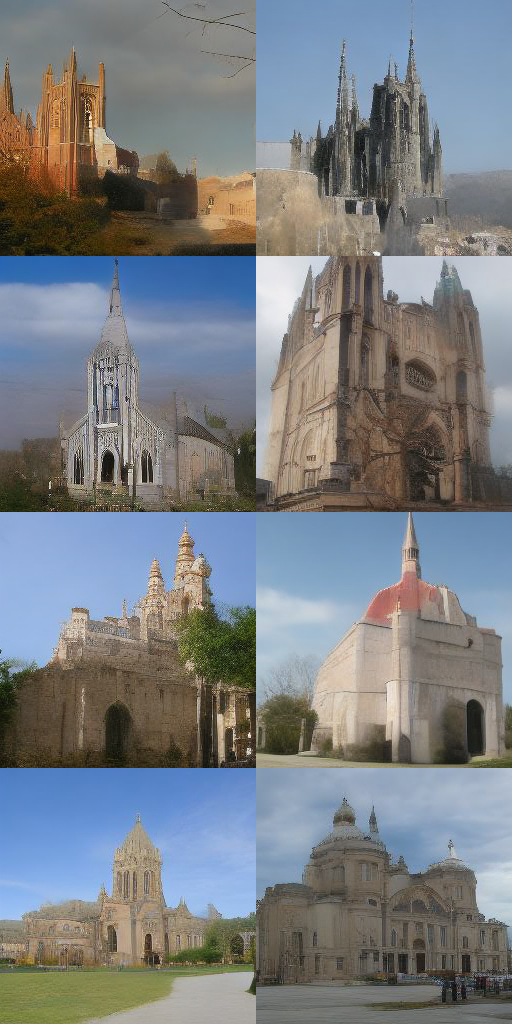}}
    & \raisebox{-.5\height}{\includegraphics[width=.18\textwidth]{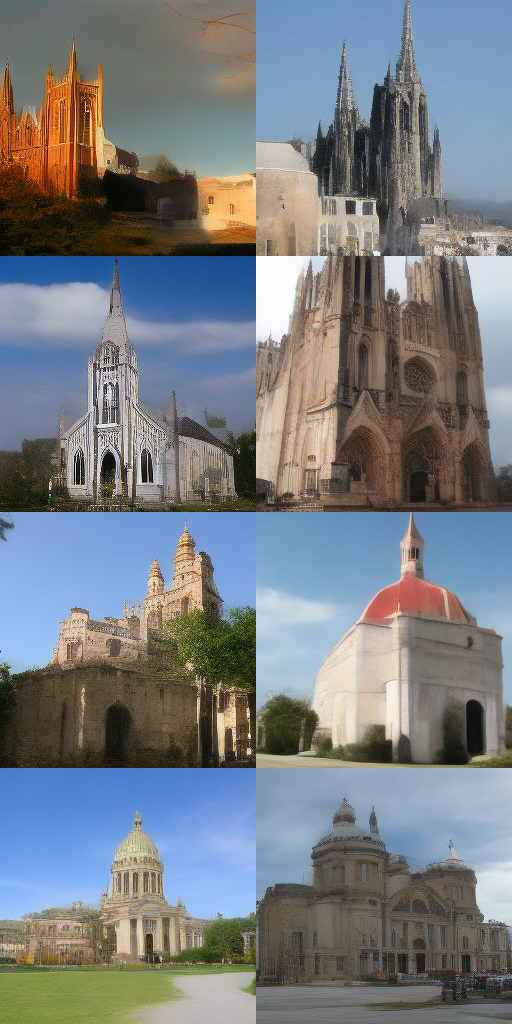}}
    & \raisebox{-.5\height}{\includegraphics[width=.18\textwidth]{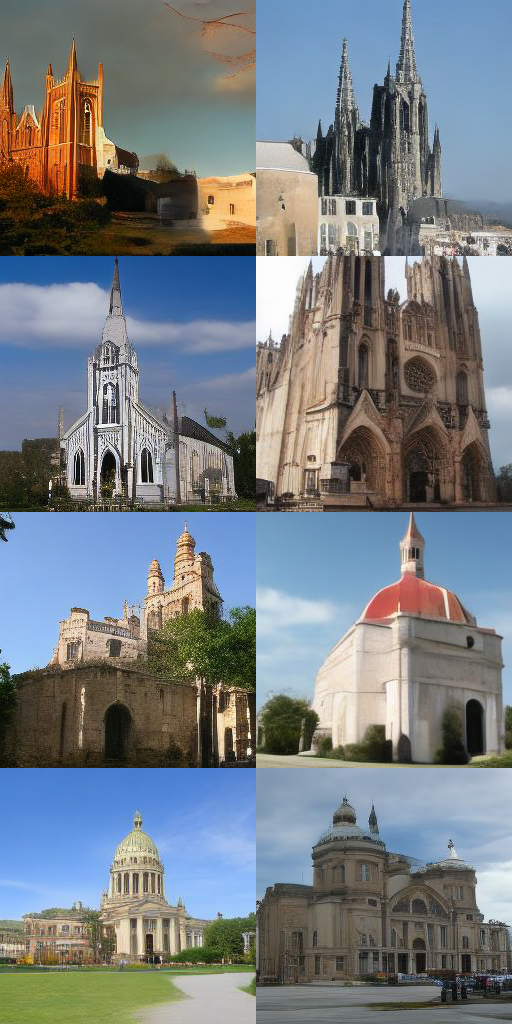}} \\\bottomrule
    \end{tabu}
    \end{adjustbox}
    \caption{Comparisons of unconditional sampling results across different datasets employing DC-Solver, UniPC~\cite{zhao2023unipc}, DPM-Solver++~\cite{lu2022dpmsolverpp} and DEIS~\cite{zhang2022fast_deis}. Images are sampled using only 5 NFE.} 
    \label{fig:unconditional_image}
\end{figure*}

\begin{figure*}
    \centering
    \begin{adjustbox}{width=\textwidth}
  \begin{tabular}{C{.40\textwidth}*{4}{C{.18\textwidth}}}
    \multicolumn{5}{c}{\makecell[c]{\textbf{Stable-Diffusion 1.5}}} \\ \toprule
    Text Prompts & DPM++~\cite{lu2022dpmsolverpp} & DEIS~\cite{zhang2022fast_deis} & UniPC~\cite{zhao2023unipc} & \textbf{DC-Solver} \\\midrule

    \makecell[c]{\textit{``A realistic photo of a tropical} \\ \textit{rainforest with diverse wildlife.''}}
    & \raisebox{-.5\height}{\includegraphics[width=.74\textwidth]{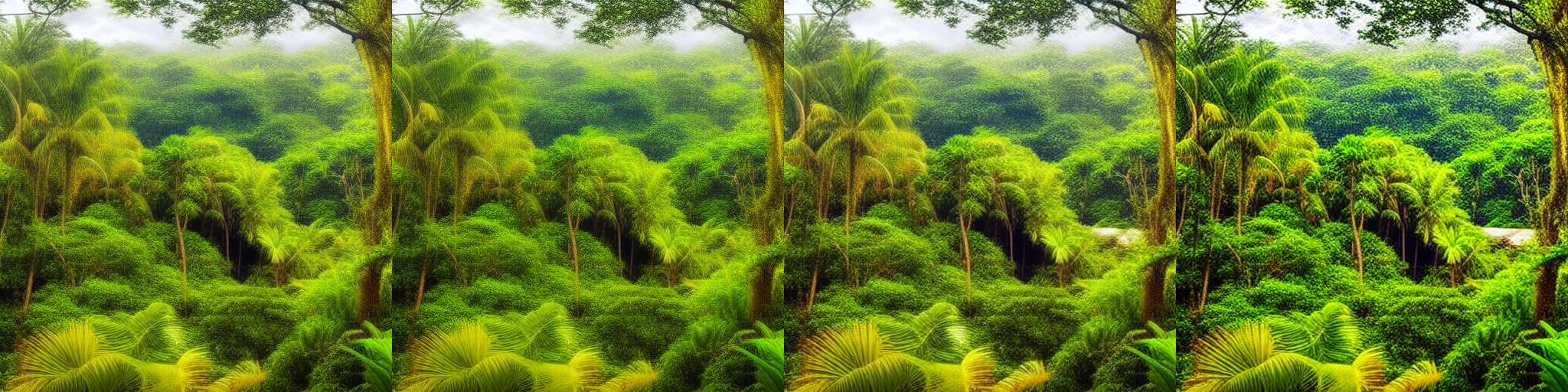}} \\ \midrule

    \makecell[c]{\textit{``Close up of a teddy bear} \\ \textit{sitting on top of it.''}}
    & \raisebox{-.5\height}{\includegraphics[width=.74\textwidth]{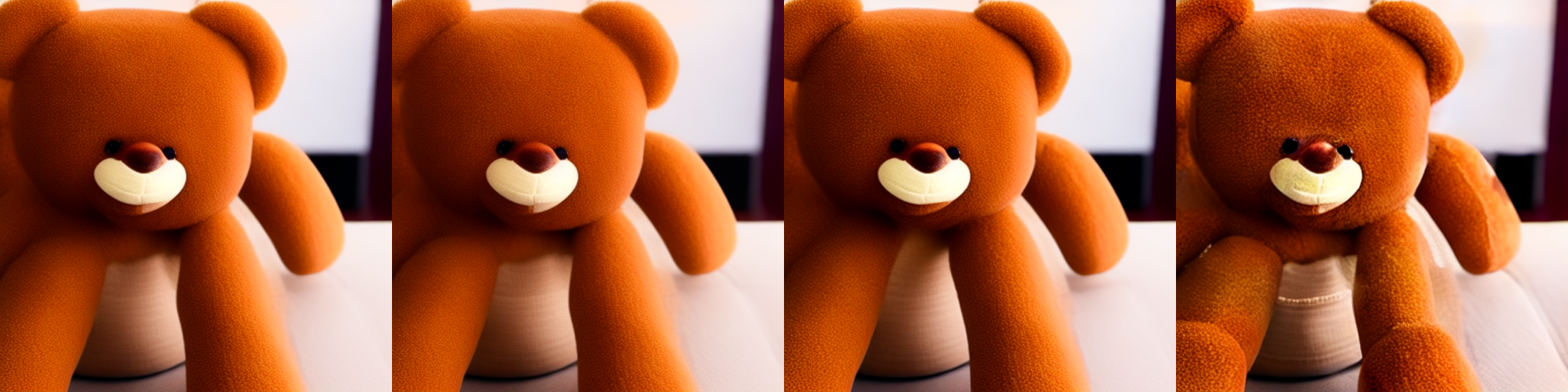}} \\ \bottomrule
  \end{tabular}
  \end{adjustbox}
  
  \begin{adjustbox}{width=\textwidth}
  \begin{tabular}{C{.40\textwidth}*{4}{C{.18\textwidth}}}
    \multicolumn{5}{c}{\makecell[c]{\textbf{Stable-Diffusion 2.1}}} \\ \toprule
    Text Prompts & DPM++~\cite{lu2022dpmsolverpp} & DEIS~\cite{zhang2022fast_deis} & UniPC~\cite{zhao2023unipc} & \textbf{DC-Solver} \\\midrule

    \makecell[c]{\textit{``Group of people standing on} \\ \textit{top of a snow covered slope.''}}
     &\raisebox{-.5\height}{\includegraphics[width=.74\textwidth]{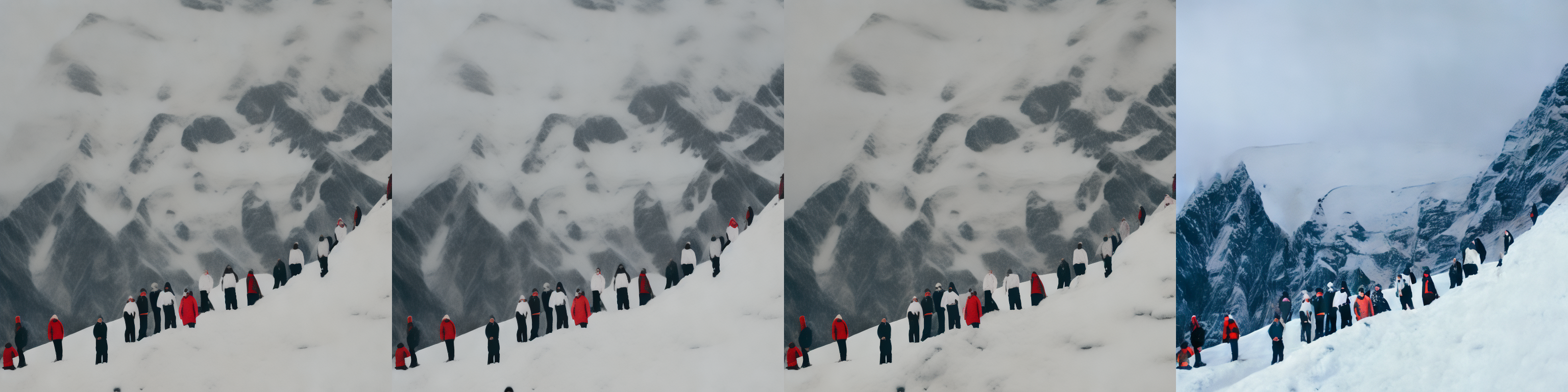}}\\\midrule
   
    \makecell[c]{\textit{``Close up of a bird perched} \\ \textit{on top of a tree.''}}
    &\raisebox{-.5\height}{\includegraphics[width=.74\textwidth]{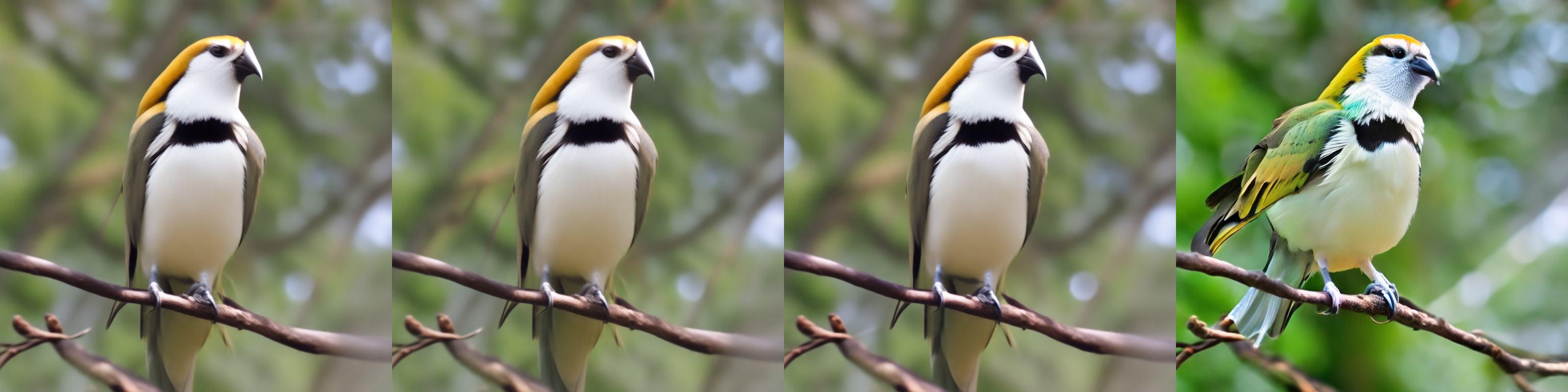}}\\\bottomrule
    \end{tabular}
    \end{adjustbox}
    
    \begin{adjustbox}{width=\textwidth}
  \begin{tabular}{C{.40\textwidth}*{4}{C{.18\textwidth}}}
    \multicolumn{5}{c}{\makecell[c]{\textbf{Stable-Diffusion XL}}} \\ \toprule
    Text Prompts & DPM++~\cite{lu2022dpmsolverpp} & DEIS~\cite{zhang2022fast_deis} & UniPC~\cite{zhao2023unipc} & \textbf{DC-Solver} \\\midrule

    \makecell[c]{\textit{``Pizza that is sitting on} \\ \textit{top of a plate.''}}
     &\raisebox{-.5\height}{\includegraphics[width=.74\textwidth]{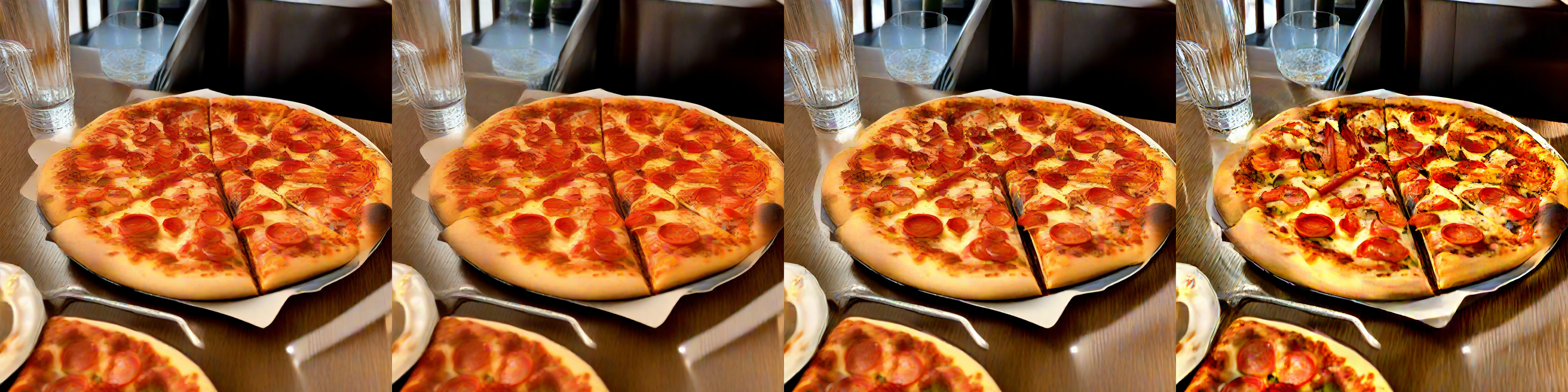}}\\\midrule
   
    \makecell[c]{\textit{``A serene waterfall in a} \\ \textit{lush green forest.''}}
    &\raisebox{-.5\height}{\includegraphics[width=.74\textwidth]{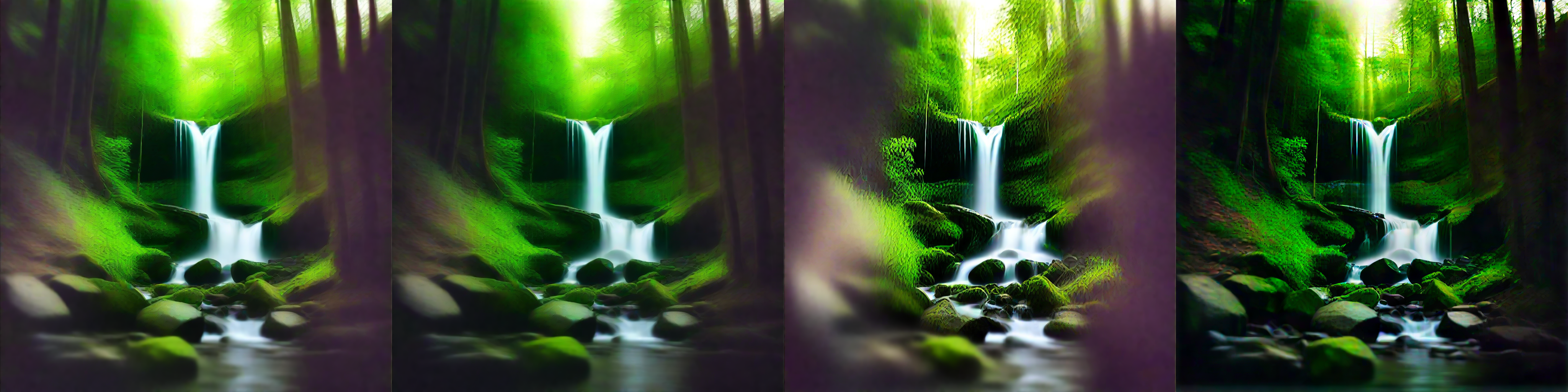}}\\\bottomrule
    \end{tabular}
    \end{adjustbox}

    \caption{Comparisons of text-to-image results on different pre-trained Stable-Diffusion models using DC-Solver, UniPC~\cite{zhao2023unipc}, DPM-Solver++~\cite{lu2022dpmsolverpp} and DEIS~\cite{zhang2022fast_deis}. Images are sampled with a classifier-free guidance scale 7.5, using only 5 NFE.} 
    \label{fig:conditional_image}
\end{figure*}

\begin{figure*}[!t]
    \includegraphics[width=\linewidth]{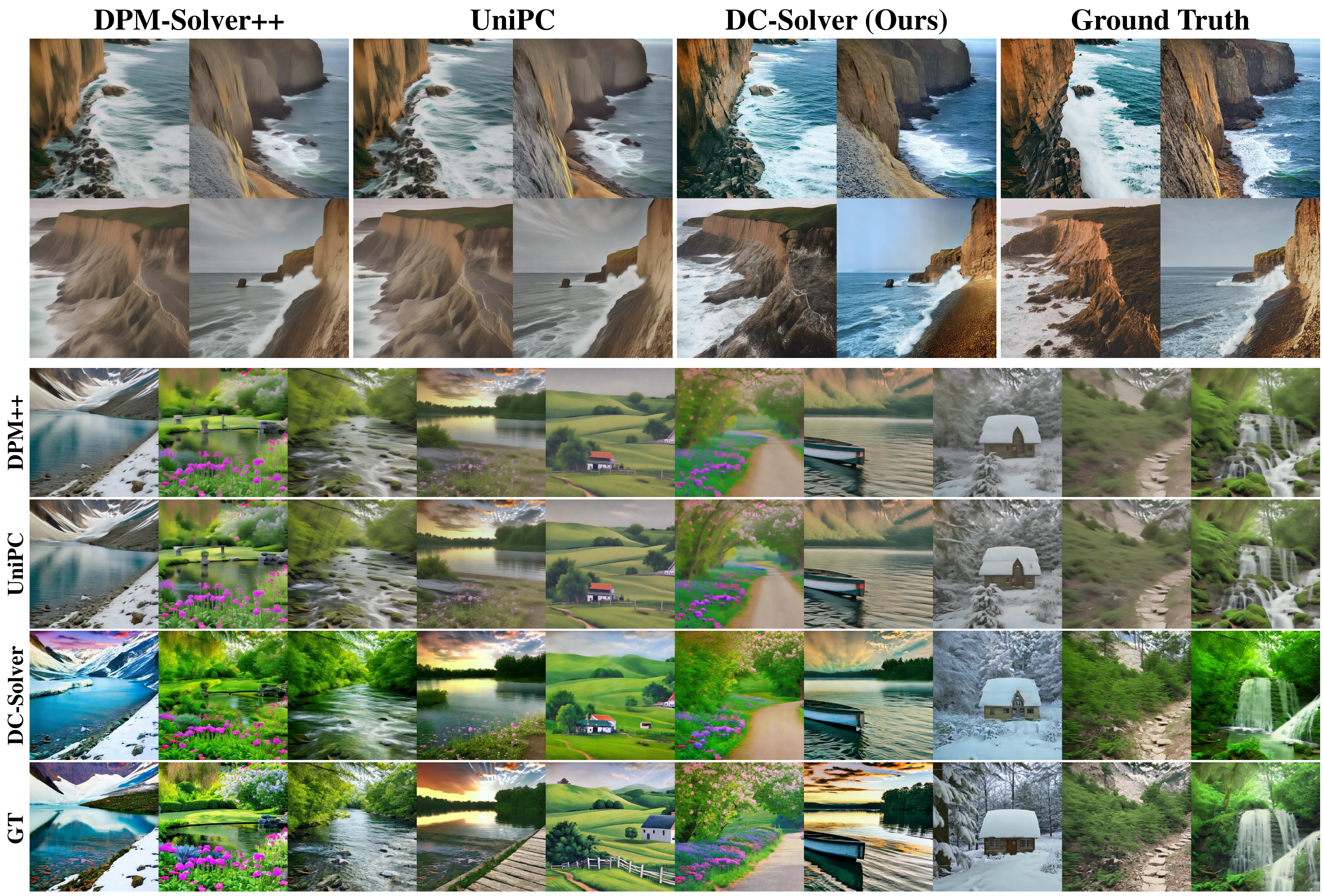}
    \caption{Comparison with GT (upper part) and more uncurated results (lower part). For all the compared methods, we adopt NFE=5 and use the same initial noise. We can clearly find that DC-Solver outperforms other methods.}
\end{figure*}

\subsection{Qualitative Results}\label{app:more_viz}
We present additional visualizations to showcase the superior qualitative performance of DC-Solver in both unconditional sampling and conditional sampling. Initially, we compare the unconditional sampling quality of four different methods on FFHQ~\cite{karras2019ffhq}, LSUN-Church~\cite{yu2015lsun} and LSUN-Bedroom~\cite{yu2015lsun} in~\Cref{fig:unconditional_image}, employing only 5 NFE. We show that DC-Solver can produce the clearest and most realistic images across all three datasets. Furthermore, we explore conditional sampling on different pre-trained Stable-Diffusion(SD) models, including SD1.5, SD2.1 and SDXL, with only 5 NFE. The reuslts in \Cref{fig:conditional_image} demonstrate that our DC-Solver is able to generate more realistic images with more details, consistently outperforming other methods on all three SD models.

\section{Implementation Details}
Our DC-Solver is built on the predictor-corrector framework UniPC~\cite{zhao2023unipc} by default. We set the order of the dynamic compensation $K=2$ and skip the compensation when $i<K$, which is equivalent to $\rho_0=\rho_1=1.0$. $K=2$ also implies a parabola-like interpolation trajectory. During the searching stage, we set the number of datapoints $N=10$. We use a 999-step DDIM~\cite{song2020denoising_ddim} to generate the ground truth trajectory $\bs x_t^{\rm GT}$ in the conditional sampling while we found a 200-step DDIM is enough for unconditional sampling. We use AdamW~\cite{adamw} to optimize the compensation ratios for only $L=40$ iterations and set the learning rate of learnable parameters as $\alpha=0.1$. For the cascade polynomial regression, we use $p_1=p_2=2$ and $p_3=3$. For the experiments on Latent-Diffusion~\cite{rombach2022high}, we adopt their original checkpoints and use the default latent size 64$\times$64. For the experiments of conditional sampling using Stable-Diffusion~\cite{rombach2022high}, we use the default latent size of 64$\times$64, 64$\times$64,  96$\times$96, 128$\times$128 for SD1.4, SD1.5, SD2.1, SDXL, respectively. It is worth noting that our method can be scaled up to larger latent sizes and pre-trained DPMs mainly because of the effectiveness of the designed dynamic compensation, which can be controlled by several scalar parameters.


\end{appendix}

\bibliographystyle{splncs04}
\bibliography{main}

\end{document}